%% file: acl_latex.tex
\documentclass[11pt]{article}

\usepackage[preprint]{acl}

\usepackage{times}
\usepackage{latexsym}

\usepackage[T1]{fontenc}
\usepackage[utf8]{inputenc}

\usepackage{microtype}

\usepackage{inconsolata}

\usepackage{graphicx}
\usepackage{amsfonts}
\usepackage{amsmath}
\usepackage{amssymb}
\usepackage{mathtools}
\usepackage{amsthm}
\newtheorem{theorem}{Theorem}
\usepackage{booktabs}
\usepackage{algorithm}
\usepackage{algorithmic}

\title{Tree SAE: Learning Hierarchical Feature Structures in Sparse Autoencoders}

\author{
  \textbf{Tue M. Cao\textsuperscript{1,2}},
  \textbf{Hoang X. Nhat\textsuperscript{2}},
  \textbf{Raed Alharbi\textsuperscript{3}},
  \textbf{Phi Le Nguyen\textsuperscript{1}},
  \textbf{My T. Thai\textsuperscript{2}}
\\
\\
  \textsuperscript{1}Hanoi University of Science and Technology, Hanoi, Vietnam,
  \textsuperscript{2}University of Florida, Florida, USA,
\\
  \textsuperscript{3}Computer Science Department, Saudi Electronic University
\\
  \small{
    \textbf{Correspondence:} \href{mailto:mythai@cise.ufl.edu}{mythai@cise.ufl.edu}
  }
}

\begin{document}
\maketitle
\begin{abstract}
Learning hierarchical features in Sparse Autoencoders (SAEs) is essential for capturing the structured nature of real-world data and mitigating issues like feature absorption or splitting. Existing works attempt to identify hierarchical relationships within independent feature sets by relying on \textit{activation coverage}, the assumption that child feature should only activate when its parent feature activates. However, we demonstrate that this condition alone is insufficient; that is, it often produces false positives where parent and child concepts are semantically unrelated. To address this, we introduce a novel {\em reconstruction condition} that enforces a deeper functional link between hierarchical levels. By combining both activation and reconstruction constraints, we propose the Tree SAE, a model designed to learn hierarchical structures directly from within the feature set. Our results demonstrate that Tree SAEs significantly surpass the existing SAEs at learning hierarchical pairs while maintaining competitive performance to the state-of-the-art on several key benchmarks. Finally, we demonstrate the practical utility of our Tree SAE in mapping the geometry of child feature subspaces and uncovering the complex hierarchical concept structures encoded within language models. \url{https://github.com/tue147/tree-sae}
\end{abstract}

\section{Introduction}
\label{sec: introduction}
\input{sections/1-intro}

\section{Motivation}
\label{sec: motivation}
\input{sections/2-motivation}

\section{Tree Sparse Autoencoder}
\label{sec: methodology}
% \input{sections/3-methodology}
\input{sections/newSection4}

\section{Experiments}
\label{sec: experiment}
\input{sections/4-experiment}

\section{Conclusion}
\label{sec: discussion}
\input{sections/5-discussion}

\bibliography{custom}

\appendix
\label{sec: appendix}
\input{sections/appendix}

\end{document}

%% file: sections/1-intro.tex
Sparse autoencoders (SAEs) have shown great promise in extracting human-interpretable features from opaque large language models, providing valuable insights in understanding the thought process \cite{monosemanticity, cunningham2023sparse, lieberum2024gemma}. Applications include model steering \cite{sae_steering}, tracing thoughts \cite{feature_circuit, transcoder_circuit}, or studying the representation space \cite{geometry, hierarchical_geometry, not_all_feature_linear}. 

%Despite its potential, SAEs suffer from imperfect feature recovery, as they neglect the inherent hierarchical structure of real-world features \cite{absorption, monosemanticity, matryoshka, metasae, mdlsae}. This hampers the studying of feature spaces and model representations~\cite{geometry, refusal_geometry}, especially since there is evidence suggesting that the structure %and organization 
%of concepts in language models are hierarchical~\cite{hierarchical_geometry}. In addition, failure in learning hierarchical features can lead to feature absorption \cite{absorption} where fine-grained features compete with and inhibit coarse-grained features in specific instances, feature splitting \cite{monosemanticity}, and feature composition \cite{metasae, matryoshka} where the SAE fails to learn generalized concept structure and instead learns fragmented or composite concepts. 

However, standard SAEs often struggle with imperfect feature recovery because they treat features as an independent set, neglecting the inherent hierarchical structure of real-world concepts \cite{absorption, mdlsae, matryoshka}. This oversight is particularly problematic given that language models appear to organize concepts hierarchically \cite{hierarchical_geometry}. When SAEs ignore this structure, they fall prey to feature absorption (where fine-grained features inhibit coarse ones), feature splitting \cite{monosemanticity}, and feature composition (where the model learns fragmented or overly complex ``polysemantic" amalgams) \cite{metasae, matryoshka}. More discussion of these undesirable phenomena is provided in Appendix \ref{sec: related works}.

Recent methods like Matryoshka SAEs \cite{matryoshka} and Matching Pursuit SAEs (MP-SAE) \cite{mp} introduced multi-level reconstruction and conditional orthogonality, which partially addressed feature splitting, absorption and enabled the learning of hierarchical features.
% While these models reduce splitting and absorption, they still share a fundamental limitation: they lack explicit structure within the feature set. 
While these works claim to capture hierarchical pairs, their method of identifying parent-child pairs rely on
% Existing hierarchical SAEs learn independent features and rely on post-hoc algorithms like Masked Cosine Similarity (MCS) \cite{monosemanticity} to identify parent-child pairs based on activation coverage, the requirement that a child feature activates only when its parent is active.
%post-hoc algorithms such as 
Masked Cosine Similarity (MCS) \cite{monosemanticity}, which is based on \textit{activation coverage}, the requirement that a child feature activates only when its parent is active.

In this work, we show that activation coverage alone is insufficient. In particular, we provide evidence that a pair of features can satisfy this condition (hence deemed parent-child) while being semantically unrelated. We further show that, indeed, previous SAEs \cite{matryoshka, mp, topk} using MCS or its variants also suffer from this phenomenon.

To solve this, we propose \textit{reconstruction condition}, which requires parent and child decoder vectors to align with the representation of the \emph{true} child concept, ensuring both features to share the same general meaning. To identify hierarchical pairs, we introduce a novel Tree SAE with a tree-like structure directly within the feature set. Specifically, features are divided into privilege layers, each with an allocation vector assigning child features to parents in lower layers. The parent and child pairs in our tree are designed to obey both activation coverage and reconstruction condition, ensuring their hierarchical relationship. 
We further propose dynamic allocation mechanism, which allows Tree SAE to assign parent-child pairs during training, fostering the exploration of hierarchical concept structure. Our contributions are:
\begin{itemize}
    \item We show that activation coverage, the current definition of parent and child relation, finds semantically incoherent pairs. We then suggest a new definition, that includes both activation coverage and our new reconstruction condition, and introduce a new metric benchmarking hierarchical learning in SAEs.
    \item We propose a novel Tree SAE that encodes, instead of independent features, a hierarchical structure within the feature set that is dynamically learned during training.
    \item Tree SAEs do not suffer from the shortcoming of activation coverage and learn significantly more coherent parent-child pairs, matching or exceeding SOTA on absorption, splitting, composition, reconstruction, and interpretability at many scale.
    \item We demonstrate the utility of explicit hierarchy in analyzing child concept subspace geometry and how language models decompose concepts across granularity levels.
\end{itemize}

%% file: sections/2-motivation.tex
% To identify hierarchical features in existing SAEs, a common approach is MCS \cite{monosemanticity, matryoshka} where the parent activation covers the child feature activation, termed as  \textit{activation coverage} condition (formal definition in Section \ref{sec: preliminary}). However, we show evidence that this condition is not sufficient to find hierarchical feature pairs (Section \ref{sec: activation coverage fail}) and provide a stronger condition to address this shortcoming (Section \ref{sec: reconstruction condition}).

In this section, we first review the standard SAE training setup that we build on. We then discuss how existing works define and discover hierarchical features, which is based on \emph{activation coverage}. Finally, we discuss its shortcomings via a detailed and novel case study, which motivates this work.

\subsection{Preliminary}
\label{sec: preliminary}

\begin{figure*}
\begin{center}
\vspace{-10mm}
\includegraphics[width=0.85\textwidth]{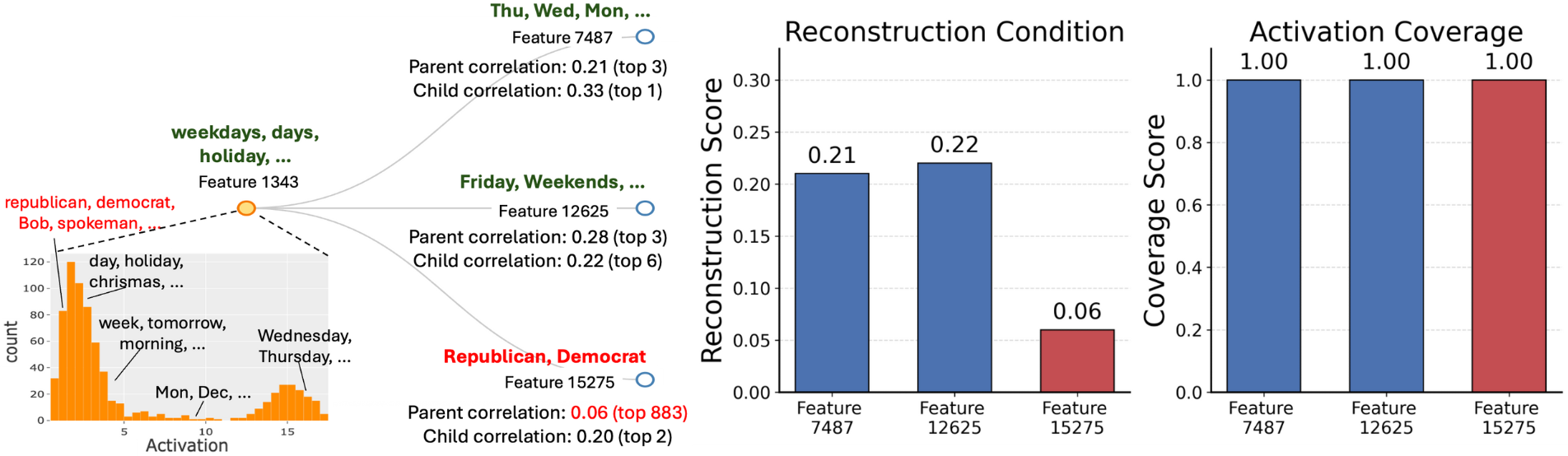} 
\end{center}
\vspace{-3mm}
\caption{The correlation between non-dense parent feature vector (feature 1343 - Tree SAE 2 layers $L_0=32$) and 3 probe weights trained to detect the activation of 3 child features (found by activation coverage condition). Child feature 15275, representing a different concept from the rest, corresponding to a spurious low activation value of the parent feature. It has significantly lower parent correlation while having perfect activation coverage. In contrast, our reconstruction condition scores correctly identify the unrelated child feature.} 
\vspace{-3mm}
\label{fig: non dense feature evidence}
\end{figure*}

\paragraph{Sparse Autoencoder.} We base on the standard SAE in \cite{monosemanticity, cunningham2023sparse}. Given the hidden latent of the model $\mathbf{x} \in \mathbb{R}^{d_m}$ at a specific layer, $\mathbf{W}^e \in \mathbb{R}^{d_f \times d_m}$ and $\mathbf{W}^d \in \mathbb{R}^{d_m \times d_f}$ are encoder and decoder weights, and bias $\mathbf{b} \in \mathbb{R}^{d_m}$, the SAE reconstruction is given as:
\begin{equation}
    \hat{\mathbf{x}} = \mathbf{W}^d \mathbf{f}(\mathbf{x}) + \mathbf{b},
\end{equation}
where
\begin{equation}
    \mathbf{f}(\mathbf{x}) = \sigma (\mathbf{W}^e (\mathbf{x} - \mathbf{b})).
\end{equation}

% Specifically, the encoder weights $\mathbf{W}^e \in \mathbb{R}^{d_f \times d_m}$ and the non-linear sparse activation function $\sigma$ maps the hidden latent of the model $\mathbf{x} \in \mathbb{R}^{d_m}$ from the original $\mathbb{R}^{d_m}$ latent space into a higher-dimensional overcomplete feature space \cite{monosemanticity} $\mathbb{R}^{d_f}$: 

% where $\mathbf{b} \in \mathbb{R}^{d_m}$ is the autoencoder bias. 

% The decoder maps the sparse features back to the original space, providing a reconstruction of the original latent:

% where $\mathbf{W}^d \in \mathbb{R}^{d_m \times d_f}$ is the decoder weights.

In the standard setting, the SAE is trained with the mean squared error reconstruction loss, plus an auxiliary loss \cite{topk} scaled with parameter $\alpha$ to avoid inactive features:
\begin{equation}
    \mathcal{L}(\mathbf{x}) = || \hat{\mathbf{x}} - \mathbf{x}||^2_2 + \alpha \mathcal{L}_{aux}(\mathbf{x}).
\end{equation}

In this setup, the activation function $\sigma$ enforces sparse \emph{non-negative} activation for each individual feature. 
To quantify the sparsity of a trained SAE, the $L_0$ of an SAE is defined as the average number of active features per sample over a given dataset. 

% This setup allows the SAE to learn in an unsupervised manner while extracting a sparse set of interpretable features from the original latent $\mathbf{x}$.

Note that standard SAE training does not impose any hierarchy among the learned features.

\paragraph{Activation coverage.} 
To formalize the notion of ``parent-child'' features, previous works rely on what we call \emph{activation coverage}. 
%Previous research \cite{matryoshka, mp, monosemanticity} identifies hierarchical feature pairs primarily through activation coverage. 
The main idea of activation coverage is that a child feature is considered a more specialized version of a parent feature, similar to an inheritance relationship in object-oriented programming. Thus, if the child feature activates on an instance, the parent needs to activate to.
% feature $f_p$ is considered a (potential) parent of a child feature $f_c$ if $f_p$ activates on a large proportion of the instances where $f_c$ is active.
This idea is embedded in an algorithm called MCS \cite{monosemanticity, matryoshka} that previous works rely on.

Mathematically, under activation coverage, two features $f_p$ and $f_c$ can be considered to have a parent-child relation if:

\begin{equation}\label{eq: activation coverage}
\frac{\sum_\mathbf{x} \mathbf{1}(f_p(\mathbf{x}) > 0 \cap f_c(\mathbf{x}) > 0)}{\sum_\mathbf{x} \mathbf{1}(f_c(\mathbf{x}) > 0)} > \tau_{cov},
\end{equation}

\noindent
where $\mathbf{1}$ is the indicator function and $\tau_{cov}$ is a pre-defined threshold \cite{matryoshka, monosemanticity}.
% Intuitively, this condition requires the parent to represent a broad concept that encompasses the more specific semantic scope of the child. 
Existing works base both their evaluation and methodology on this definition of parent-child features relationship. 

However, we argue that this definition suffers from two major issues. First, it tends to identify features that activate very frequently as parent feature of many other features, leading to overly broad parent features.
Concretely, we point out a specific failure mode in Sec~\ref{sec: activation coverage fail}, which is linked to existing observation in the literature of a ``PCA feature''.
% A densely activated or highly polysemantic parent feature may achieve a high coverage score across many children without sharing any underlying functional or semantic relationship with them.
Second, the formalization in Eq.~(\ref{eq: activation coverage}) is extremely sensitive to methodological choices and hyperparameters. For example, the condition $f_p(\mathbf{x})>0$ risks misidentifying weak activations with low values as genuine parent feature activation.
In Sec~\ref{sec: activation coverage fail}, we qualitatively show a specific example of this issue,  where activation coverage mis-identify a child feature that only activates on instances with low parent feature activations, and the instances themselves bear no meaningful connection to the supposed parent concept (Figure \ref{fig: non dense feature evidence}).
Collectively, they suggest that activation coverage is a necessary but insufficient condition for establishing parent-child feature relation.

\subsection{Case Study of Activation Coverage Failures} 
\label{sec: activation coverage fail}

\paragraph{Overly broad parent features.} We observe that activation coverage identifies spurious hierarchical pairs across many SAEs (more evidence in Appendix \ref{sec: exp activation coverage fail}). We first investigate a common dense feature called ``PCA feature" found by \cite{dense_not_bug}, of which the decoder vector has high correlation with the top-1 PCA vector of the activation space.
The PCA feature we study, feature 3098 of the 4-layer Tree SAE $L_0 = 32$, which activates on a large proportion of tokens (20\% of the test tokens in our example), does not appear to encode any interpretable meaning when analyzed at the token level, similar to the observation in \cite{dense_not_bug}. As a result of dense activation, most non-dense child features that can be interpreted through token activations have a perfect activation coverage score with the dense parent PCA feature. However, the parent and child concepts have no shared meaning due to the uninterpretability of the parent feature.

\paragraph{Low, spurrious parent activations.} An additional example demonstrating the failure of activation coverage to identify semantically meaningful hierarchical relationships, even in the presence of non-dense parent features, is illustrated in Figure \ref{fig: non dense feature evidence}. We examine parent feature 1343 from the Tree SAE (2 layers, $L_0=32$), which primarily activates on ``time" contexts such as \textit{``weekdays"}, \textit{``today"}, and \textit{``morning"}. Although the activation coverage metric identifies features 7487, 12625, and 15275 as children (shown in Figure \ref{fig: non dense feature evidence}), a qualitative inspection reveals significant semantic divergence. While features 7487 and 12625 activate in a similar context as the parent feature but are decomposed into specialised sub-cases; feature 15275 activates on political tokens, such as \textit{``Republican''} and \textit{``Democrat''}, that are entirely unrelated to the concept of the parent feature. As further evidenced by the broader analysis across all SAEs provided in Appendix \ref{sec: exp activation coverage fail}, these results suggest that relying exclusively on activation coverage often yields parent-child pairings that lack conceptual coherence.%, thereby undermining the interpretability of the learned hierarchy.

\section{ Reconstruction Condition}
\label{sec: reconstruction condition}
%To address the limitations of activation coverage, 
We introduce a reconstruction condition designed to enforce stricter semantic alignment within hierarchical feature pairs. Formally, for a parent feature $f_p$ and a potential child feature $f_c$, the reconstruction score $S_{res}(f_p, f_c)$ is defined as:
\begin{equation}
    \label{eq: reconstruction condition}
    S_{res}(f_p, f_c) = \min((\mathbf{d}^*_c)^T \mathbf{d}_c, (\mathbf{d}^*_c)^T \mathbf{d}_p),
\end{equation}
where $\mathbf{d}_c, \mathbf{d}_p \in \mathbb{R}^{d_m}$ denote the decoder vectors corresponding to $f_c$ and $f_p$, respectively; and $\mathbf{d}^*_c$ represents the true concept vectors in activation space associated with the of $f_c$.
Intuitively, a valid hierarchical relationship requires that both parent and child features contribute constructively to the reconstruction of the child concept. This ensures that the feature vector of the parent captures broad conceptual direction, which the child’s vector then further refines toward a more specialized sub-concept. Consequently, we argue that robust identification of hierarchical pairs requires the simultaneous satisfaction of two criteria: activation coverage and the reconstruction condition. While activation coverage ensures the parent concept is sufficiently broad to encompass the child’s occurrences, the reconstruction condition enforces semantic alignment.

% Intuitively, we want both the parent and child features in reconstructing the child concept. This ensures that the feature vector of the parent actively represents a general concept direction, and the child's vector further refines the direction toward the specialized child concept. We argue that to determine hierarchical feature pairs, we need the both conditions: \textit{activation coverage} and \textit{reconstruction}. Specifically, the \textit{activation coverage} ensures the parent concept is broad enough to cover the child concept, while the \textit{reconstruction condition} tightens the concept between parent and child features, avoiding the problem of unrelatedness from activation coverage.
% Specifically, given the true concept vector of the child concept represented by the activation of a candidate child feature $f_c$: $\mathbf{d}^*_c \in \mathbb{R}^{d_m}$, we define the \textit{reconstruction condition} of $f_c$ with respect to a parent feature $f_p$ as:
% \begin{equation}
%     \label{eq: reconstruction condition}
%     S_{res}(f_p, f_c) = min((\mathbf{d}^*_c)^T \mathbf{d}_c, (\mathbf{d}^*_c)^T \mathbf{d}_p),
% \end{equation}
% where $\mathbf{d}_c$ and $\mathbf{d}_p$ are the decoder vectors of the child and parent features, respectively.
% Now, come back to the failure casess 
% examples described in the previous 

\paragraph{Revisiting activation coverage failure cases.} We now demonstrate that our proposed reconstruction condition effectively eliminates  spurious associations indentified in Section \ref{sec: activation coverage fail}. To evaluate the \textit{``PCA feature"} case, we adopt the methodology of \cite{absorption}, training a linear probe to distinguish the activations of a child feature from all other tokens. The resulting weight vector of the linear probe is utilized as the ground-truth concept direction, $\mathbf{d}^*_c$ (we provide further discussion on why we need to train probe following previous works to find $\mathbf{d}^*_c$ in Appendix \ref{sec: reconstruction loss and activation coverage improves reconstruction score}). For the initial \textit{``PCA feature"} example, we identify the top 10 features with the highest activation coverage scores and calculate the cosine similarity between both parent and child features relative to this true direction. As illustrated in Figure \ref{fig: dense feature evidence}, all identified child features exhibit a notably low correlation with the \textit{``PCA feature''}, ranking only within the top 20k features by correlation. This confirms that the reconstruction condition correctly reflects the semantic vacancy of the parent feature relative to its putative children.

The efficacy of the reconstruction score is further validated by the case of feature 1343 in Figure \ref{fig: non dense feature evidence}. While features 7487 and 12625 exhibit high reconstruction scores relative to the parent, feature 15275 yields a significantly lower value. This discrepancy demonstrates that the reconstruction condition effectively filters out spurious pairings where parent activations are either incidental or semantically divergent from the core concept. Specifically, for feature 15275, the diminished score captures a fundamental lack of conceptual correlation, correctly identifying it as unrelated to the parent’s temporal domain. In contrast, features 7487 and 12625 maintain some of the highest parent-child correlations in the dataset, resulting in robust reconstruction scores. This confirms that our metric successfully isolates features that serve as specialized refinements of a broader parent concept. We also illustrate why both conditions are needed in Appendix \ref{sec: necessity of both}.

\paragraph{Experiment beyond a few examples.} We show that, in general, strong hierarchical pairs have a higher reconstruction score than weaker pairs in Appendix \ref{sec: exp activation coverage fail}, suggesting that it is natural for hierarchical pairs to follow this condition.

% Lastly, in our experiments (Section \ref{exp: hierarchy}), we show that applying activation coverage on the existing SAEs leads to a poor relation between parent and child feature for most of the identified hierarchical pairs.

%% file: sections/newSection4.tex
%Just some skeleton, check it carefully

% {\bf NEW SECTION 4 SKELETON}
% \textcolor{red}{Tue: Please carefully read my new skeleton and correct and expand/edit as needed
%  but that is the flow to show more meat and emphasize more on the dynamic part. Indeed, the most novelty in this alg is the dynamic allocation part.}

{\bf Overview.} In this section, we detail our Tree SAE, which is the first SAE designed to directly model hierarchical, real-world concept structures within its feature set. Unlike standard SAEs that learn independent features, our architecture relies on three core mechanisms: (1) an explicit tree structure combined with an activation gating mechanism to enforce \textit{activation coverage}; (2) a multi-level reconstruction loss to enforce the \textit{reconstruction condition}; and (3) a dynamic reallocation algorithm to prevent dead-feature pathologies during training.

Suppose that the Tree SAE has $L$ layers. We define our proposed tree architecture and training pipeline as follows:

\begin{figure*}[tb]
\begin{center}
\vspace{-10mm}
\includegraphics[width=0.85\textwidth]{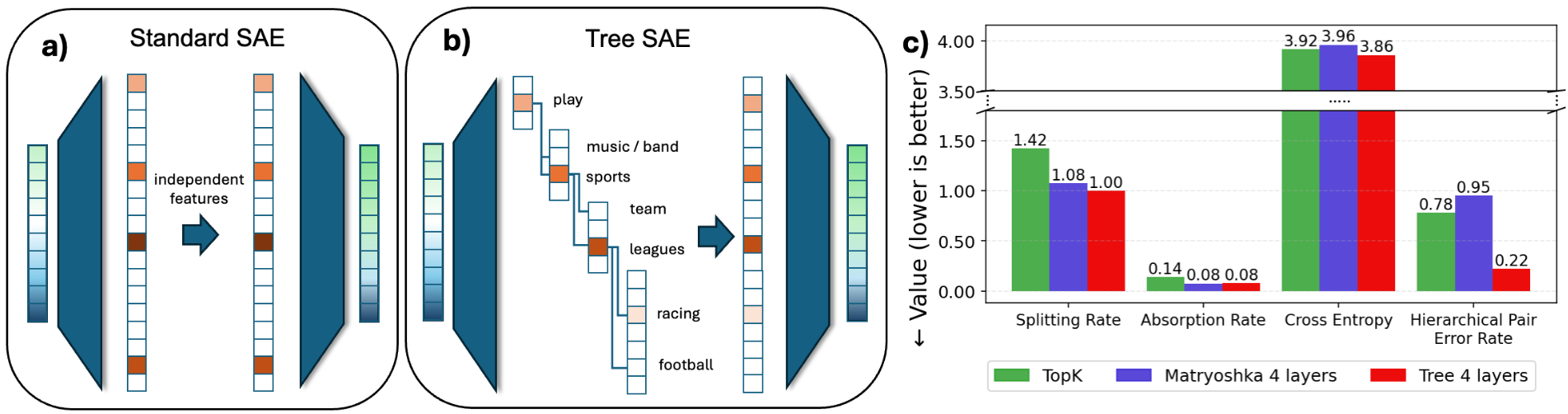} 
\end{center}
\vspace{-2mm}
\caption{\textbf{(a) and (b)} The difference between standard SAE architectures and our Tree SAE. The Tree SAE learns parent and child feature pairs over multiple layers, while other SAEs learn an independent feature set. \textbf{(c)} The results of the baseline Top-$k$, 4-layer Matryoshka, and 4-layer Tree SAE at $L_0=80$ on feature splitting, absorption, downstream cross-entropy loss, and hierarchical pair error rate (Section \ref{exp: hierarchy}). Tree SAEs perform equally well or outperform other SAEs on reconstruction and hierarchy-related metrics.} 
\vspace{-4mm}
\label{fig: main_tree_sae}
\end{figure*}

\noindent \textbf{Tree Structure and Activation Coverage.} 
We design the Tree SAE with $L$ privilege layers, organized such that a feature at a higher layer is eligible to be assigned as a descendant of a feature at any lower layer (Figure \ref{fig: main_tree_sae}). We define an allocation vector $\mathbf{a}_{l} \in \mathbb{N}^{s_l}$ containing the indices of the parent features for the $s_l$ features at layer $l$:
$\mathbf{a}_{l} := \{a_{l, 1}, \dots, a_{l, s_l}\}^{T},$
where $a_{l, i}$ denotes the parent feature of the $i$-th feature. To provide a consistent baseline for features without an explicit parent, we introduce an imaginary root node at layer 0 ($s_0=1$). This formulation allows the Tree SAE to directly encode a tree structure, unlike existing SAEs \cite{matryoshka, mp, topk} which maintain unstructured feature sets.

To enforce activation coverage, we require that a child feature activates only if its allocated parent activates, strictly following the allocation vectors $\{\mathbf{a}_{l}\}_{l=1}^L$. For a feature $f_i$ at layer $l$ with parent index $a_{l,i}$, we define the regularized activation $f^*_i(\mathbf{x})$ as:
\begin{equation}
    f^*_{i}(\mathbf{x}) = \begin{cases}
        f_{i}(\mathbf{x}) & a_{l,i} \in \text{layer 0} \\
        f_{i}(\mathbf{x}) \mathbf{1}(f^*_{a_{l,i}}(\mathbf{x}) > 0) 
       & \text{otherwise}
     \end{cases} ,
     \label{eq: enforce activation coverage}
\end{equation}
where $\mathbf{1}(\cdot)$ is the indicator function, $f^*_i(\mathbf{x})$ is the final activation for feature $f_i$ given input $\mathbf{x}$.

\noindent \textbf{Ensuring Reconstruction Condition via Training Loss.}
The reconstruction condition requires both parent and child features to simultaneously contribute to forming the child concept vector. To operationalize this, we utilize a multi-level reconstruction loss. Let $\hat{\mathbf{x}}_l$ be the reconstruction from features strictly at layer $l$. The loss constrains the cumulative reconstruction of all preceding layers, including $l$ ($\sum_{t=1}^l\hat{\mathbf{x}}_t$), to recover the input $\mathbf{x}$:
\begin{equation}
    \label{eq: multilevel_reconstruction}
    \mathcal{L}_{recons}(\mathbf{x}) = \sum_{l=1}^L || \sum_{t=1}^l \hat{\mathbf{x}}_t - \mathbf{x} ||^2_2.
\end{equation}
\textit{Intuition:} By forcing the parent and child to jointly reconstruct the input, this loss naturally pushes their decoder vectors to align constructively toward the true concept. We formally prove in Proposition 1 (Appendix \ref{sec: reconstruction loss and activation coverage improves reconstruction score}) that pairing this multi-level loss with our tree structure (Equation \ref{eq: enforce activation coverage}) mathematically guarantees high parent feature alignment with child concept, avoiding failure cases in Section \ref{sec: activation coverage fail}.

To mitigate dead features, we adopt an auxiliary loss similar to \cite{topk}. However, instead of applying it globally, we apply it per privilege layer, ensuring dead child features reconstruct the residual error of their specific parents. Let $\hat{\mathbf{e}}_l$ be the reconstruction of the residual error at layer $l$ using the top-$k$ dead features, the auxiliary loss is: $\mathcal{L}^l_{aux} \coloneqq ||\hat{\mathbf{e}}_l + \sum_{t=1}^l \hat{\mathbf{x}}_t - \mathbf{x}||^2_2,$. With scaling factor $\alpha_l$, the full training objective are defined as:
\begin{equation}
    \quad \mathcal{L}_{tree}(\mathbf{x}) = \mathcal{L}_{recons}(\mathbf{x}) + \sum_{l=1}^L \alpha_l \mathcal{L}^l_{aux}(\mathbf{x}).
\end{equation}

\noindent \textbf{Dynamic Feature Reallocation.} 
While a fixed tree provides strict hierarchy, it risks suboptimal allocations where child features are assigned to inactive parents, yielding a high dead-feature ratio. Auxiliary losses alone cannot correct these structural misallocations. To resolve this, we periodically dynamically reallocate children during training. 

Let $m_l=\sum_{t=0}^{l-1}s_t$ be the number of candidate parent features available to layer $l$, and $\mathbf{k}_l \in \mathbb{N}^{m_l}$ be the vector denoting the number of children assigned to each parent. We define a parent's \textit{capacity} $C_p$ as its accumulated training loss. A high $C_p$ indicates the parent activates frequently or reconstructs poorly, signaling a need for more child features to handle the error. If parent $f_p$ receives $k_{l, p}$ children, each child receives a \textit{payoff} of $P_p=C_p/k_{l, p}$, representing its potential to receive gradient updates. To avoid dead features, we seek the optimal allocation $\mathbf{k}^*_{l}$ that maximizes the minimum payoff across all parents:
\begin{multline}
    \label{eq: dynamic alloc}
    \mathbf{k}^*_{l} = \arg \max_{\mathbf{k}_{l}} \left(\min_{p| k_{l, p} > 0} \left(\frac{C_p}{k_{l,p}} \right) \right), \\ \textbf{s.t.} \: \sum_{p=1}^{m_l}k_{l, p}=s_l.
\end{multline}
The existence of a solution is guaranteed by the following theorem (proof in Appendix \ref{sec: dynamic allocation detail}):
\begin{theorem}
\label{theorem: threshold feasibility}
For any $\tau > 0$, there exists an allocation $\mathbf{a}_l$ with $k_{l, 1}, \dotsc, k_{l, m_l} \in \mathbb{Z}_{\geq 0}$, $\sum_p k_{l,p} = s_l$, such that $\min_{p|k_{l, p} > 0} (C_p / k_{l, p}) \geq \tau$ if and only if $\sum_{p=1}^{m_l} \left\lfloor \frac{C_p}{\tau} \right\rfloor \geq s_l.$
\end{theorem}

Based on this theorem, we find the supremum of $\tau$ efficiently and optimal $\mathbf{k}_l^*$ in $O(s_l \log m_l)$ with greedy Algorithm \ref{alg: dynamic alloc}. The allocation procedure is sketched below:

\vspace{1mm}
\noindent \textbf{Algorithm \ref{alg: full algo} Sketch: Dynamic Reallocation} (Run periodically every $T$ batches)
\begin{enumerate}
    \setlength{\itemsep}{0pt}
    \item \textbf{Compute Capacities:} For layer $l$, aggregate the batch loss for every instance a parent feature activates to define its capacity $C_p$.
    \item \textbf{Determine Optimal $\mathbf{k}^*_l$:} Solve Equation \ref{eq: dynamic alloc} using a greedy Algorithm \ref{alg: dynamic alloc}. Iteratively assign children to the parent with the currently highest payoff $C_p / (k_{l,p} + 1)$ until all $s_l$ children are allocated.
    \item \textbf{Reallocate Dead Children:} Identify dead child features at layer $l$ (inactive for $>$10M tokens). Reassign their parent indices $a_{l, i}$ using a first-fit strategy to fulfill the optimal target allocation $\mathbf{k}^*_l$.
\end{enumerate}
\vspace{1mm}
This dynamic adaptation allows the Tree SAE to continuously explore valid hierarchical structures while maximizing model capacity. Experiments, pseudo-code, and proofs showing the effectiveness of the allocation are provided in Appendix \ref{sec: dynamic allocation detail}.

%% file: sections/4-experiment.tex
% We evaluate our Tree SAE on hierarchical-related metrics, including feature splitting (Section \ref{exp: absorption and splitting}), feature absorption (Section \ref{exp: absorption and splitting}), and feature composition (Section \ref{exp: composition}), showing that Tree SAEs are not susceptible to existing hierarchical feature pathologies.
% We then investigate the ability to learn and identify hierarchical structure, based on both activation coverage and reconstruction condition, where Tree SAEs discover significantly more hierarchical pairs compared to existing SAEs in Section \ref{exp: hierarchy}.
% Furthermore, we conduct quality of reconstruction, cross entropy explained in Section \ref{exp: reconstruction loss}, and interpretability experiments in Section \ref{exp: interp}. 
% Lastly, following previous state-of-the-art \cite{matryoshka}, we test the scaling of Tree SAEs on all of the mentioned metrics in Section \ref{exp: scalling}. The overall results show that Tree SAE can reach a stronger reconstruction loss than Top-$k$ SAE and maintain its performance at different scales, comparable to Matryoshka SAE.
We evaluate Tree SAEs on hierarchical metrics, feature splitting, absorption, and composition (Section \ref{exp: absorption and splitting and composition}), showing they avoid known hierarchical feature pathologies. We then assess their ability to learn and identify hierarchical structure using activation coverage and reconstruction condition, where Tree SAEs discover substantially more hierarchical pairs than existing SAEs (Section \ref{exp: hierarchy}). Next, we evaluate reconstruction quality via cross-entropy explained and perform interpretability analyses in Section \ref{exp: reconstruction loss}. Finally, following prior state-of-the-art \cite{matryoshka}, we study scaling behavior on all metrics (Section \ref{exp: scalling}), finding that Tree SAEs achieve lower reconstruction loss than Top-$k$ SAEs and maintain strong performance across scales, comparable to Matryoshka SAEs.

% In our experiments, we implement four main SAEs, namely: our Tree SAE, Matryoshka SAE \cite{matryoshka}, Matching Pursuit SAE \cite{mp}, and the baseline Top-$k$ SAE \cite{topk}. If not mentioned otherwise, all of the SAEs have 24k features on GPT2-small \cite{gpt2} at layer 5 with the same set of training hyperparameters, provided in Appendix \ref{sec: sae training}. We use the same Top-$k$ activation function for Tree SAE and Matryoshka to be consistent with Top-$k$ SAE, with four different $L_0$ levels ranging from 32 to 80. While for MP-SAE, we follow \cite{mp} to only fix the maximum number of active features at each instance, allowing MP-SAE to learn an adaptive sparsity level. We therefore report the average $L_0$ level of MP-SAE across 1280k tokens in our experiments. To compare the quality of Tree SAEs at different depths, we implement two layers and four layers Tree (two or four privilege layers) and Matryoshka SAEs (two or four prefixes \cite{matryoshka}), where the number of feature per layer is $[6144, 18432]$ and $[1536, 3072, 9216, 10752]$ respectively. These numbers were not cherry-picked and are the multiplications of the number of dimensions of the hidden latent in GPT2-small. The average sparsity level is kept the same for all of the layers between Matryoshka and Tree SAE to ensure fairness. The sparsity level at each layer, as well as other Tree SAE-specific hyperparameters, are presented in Appendix \ref{sec: sae training}.
In our experiments, we implement four main SAEs, namely: our Tree SAE, Matryoshka SAE \cite{matryoshka}, Matching Pursuit SAE \cite{mp}, and the baseline Top-$k$ SAE \cite{topk}. If not mentioned otherwise, all of the SAEs have 24k features on GPT2-small \cite{gpt2} at layer 5. We use the same Top-$k$ activation function for Tree SAE and Matryoshka to be consistent with Top-$k$ SAE, with four different $L_0$ levels ranging from 32 to 80. While for MP-SAE, we follow \cite{mp} to only fix the maximum number of active features at each instance, allowing MP-SAE to learn an adaptive sparsity level. We therefore report the average $L_0$ level of MP-SAE across 1280k tokens in our experiments. To compare the quality of Tree SAEs at different depths, we implement two layers and four layers Tree (two or four privilege layers) and Matryoshka SAEs (two or four prefixes \cite{matryoshka}), where the average sparsity for each layer in Tree SAE and Matryoshka are kept the same to ensure fairness. The hyperparameter details for each privilege layer, and other Tree SAE-specific hyperparameters, are presented in Appendix \ref{sec: sae training}.

\begin{figure*}[tb]
\begin{center}
\vspace{-10mm}
\includegraphics[width=0.95\textwidth]{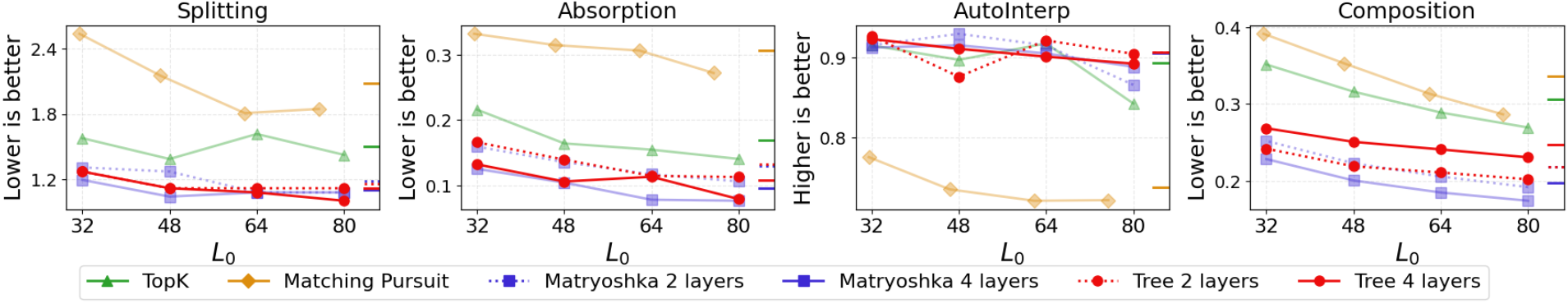} 
\end{center}
\vspace{-3mm}
\caption{The results of Feature Splitting, Absorption, AutoInterp, and Composition, respectively. The averages across all $L_0$ levels are marked on the right of the plots. 
} 
\vspace{-3mm}
\label{fig: main_hierarchy}
\end{figure*}

\begin{figure*}[tbh]
    \centering
    % \vspace{-3mm}
    \begin{minipage}{.49\textwidth}
        \centering
        \includegraphics[width=0.9\columnwidth]{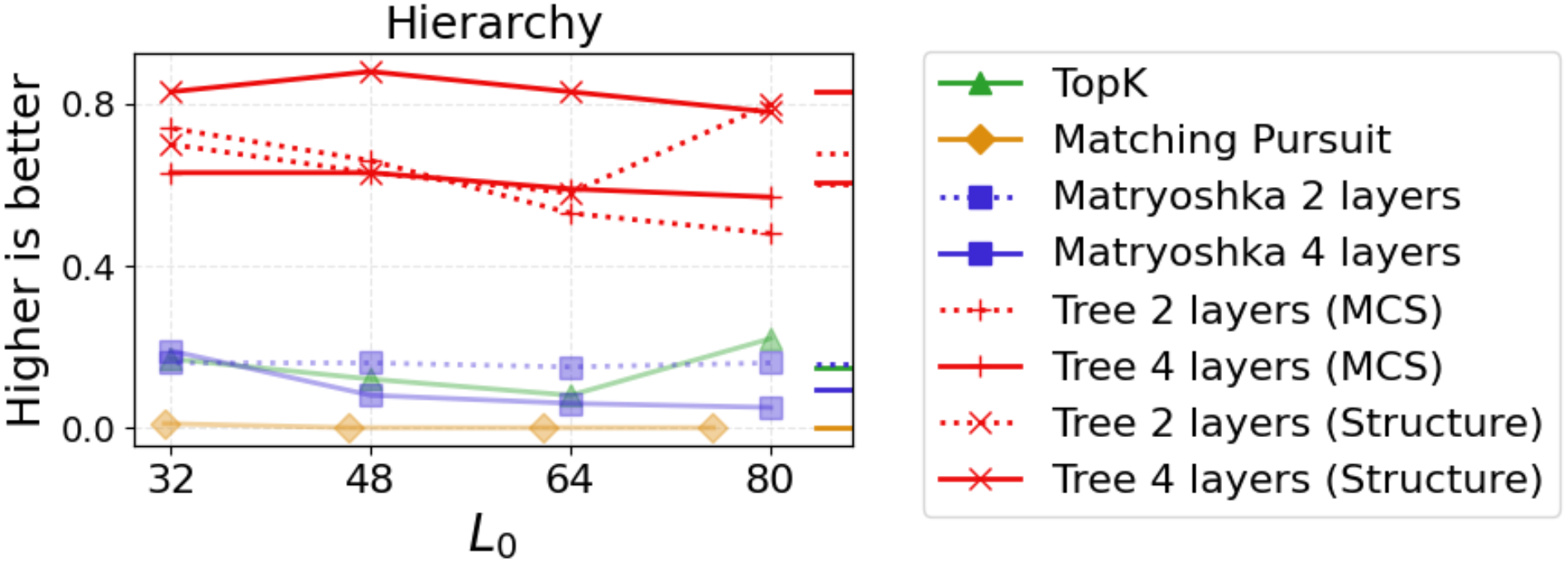}
        \vspace{-1mm}
        \caption{The result of the Hierarchy metric of the SAEs across all $L_0$ levels. The averages are marked on the right of the plot. 
        } 
    \label{fig: compare tree}
    \end{minipage}
    \hfill
    \begin{minipage}{.49\textwidth}
        \centering
        \includegraphics[width=0.9\columnwidth]{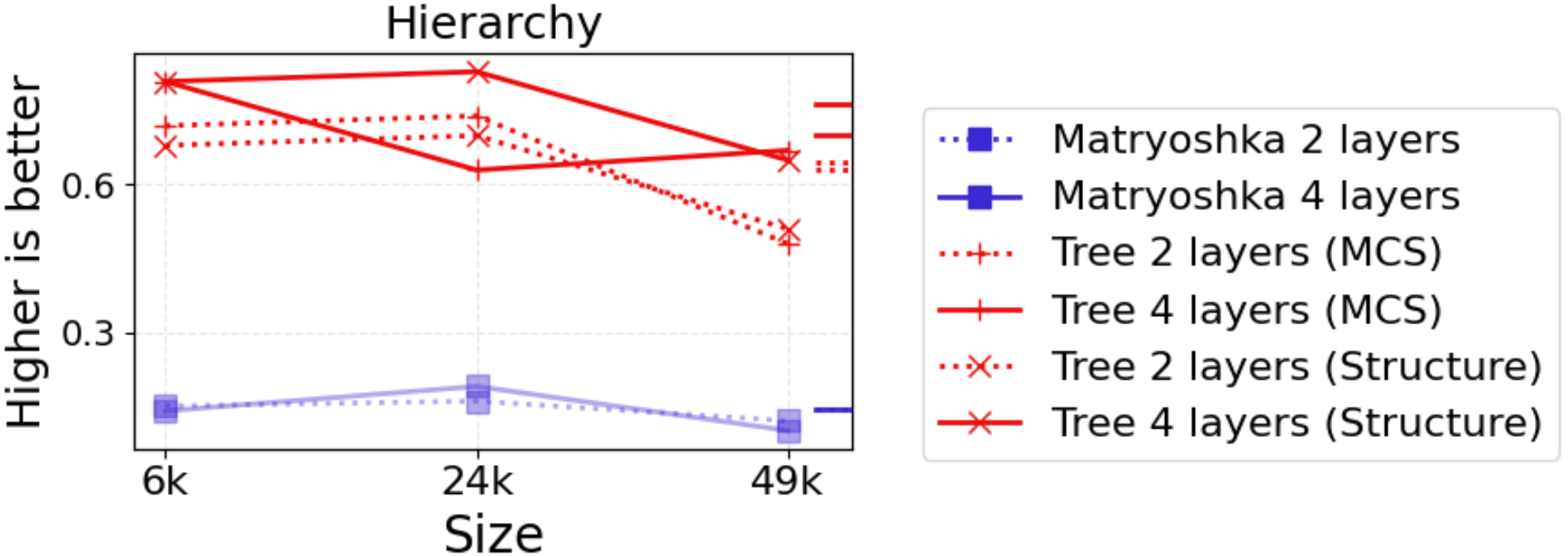} 
        \caption{The result of the Hierarchy metric of the SAEs across all sizes. The averages are marked on the right of the plot. 
        }
    \label{fig: scale compare tree}
    \end{minipage}
\end{figure*}

\subsection{Absorption, Splitting, Composition Metric}
\label{exp: absorption and splitting and composition}
We benchmark on feature absorption, splitting, and composition metrics from \cite{absorption, saebench, matryoshka}, specific setups are provided in Appendix \ref{sec: absorption and splitting setup}.

\textbf{Feature Absorption and Splitting.}
The results are shown in Figure \ref{fig: main_hierarchy}, show that Tree SAEs remain competitive with the previous SOTA Matryoshka SAEs. Specifically, we find that both Tree SAEs and Matryoshka SAEs reach strong performance on both Feature Splitting and Absorption compared to Top-$k$ SAE, consistent across the sparsity level and number of layers. On the other hand, MP-SAEs perform significantly worse than Top-$k$ SAE on both metrics.

%\subsection{Feature Composition}
%\label{exp: composition}
% We follow \cite{matryoshka} in quantifying feature composition and shared information between the feature set; setup details are in Appendix \ref{sec: composition setup}.
\textbf{Feature Composition.} The results in Figure \ref{fig: main_hierarchy} indicate that Tree SAEs are comparable to the SOTA Matryoshka SAEs. In particular, we observe that, while Tree SAEs perform slightly worse, the Tree and Matryoshka SAEs substantially have a lower Feature Composition rate compared to other SAEs. Furthermore, the MP-SAE features are noticeably more similar to each other in the encoded information than Top-$k$ SAE. 

\subsection{Learning Hierarchical Feature}
\label{exp: hierarchy}
Given the activation coverage and reconstruction condition, we conduct a quantitative experiment to find which SAE is able to learn more hierarchical feature pairs, and which procedure (our Tree SAE structure or MCS \footnote{We discuss the best variant of MCS in Appendix \ref{sec: choice of MCS}.} \cite{matryoshka, monosemanticity}) is more effective at searching for hierarchical structure. Our observations indicate that Tree SAEs strongly follow both conditions on both procedures, posing a significant gap with the rest of the SAEs.

% For each SAE, we sample 100 random parent features that are above the 50\% most dense feature (we sample denser features as they are more likely to have children). We use MCS or Tree Structure to identify the top-5 child features for each parent. For Tree SAE, which can run both procedures, we sample an equal number of children for each procedure (for example, if a parent has 7 children based on our Tree Structure, we also sample the top 7 children using MCS). Then, we train a linear probe to obtain the true direction for each child feature. We report the number of cases where both parent and child features are among the top-5 features with the highest correlation with the true direction, normalised by the total number of child features. Reaching a higher rate means that the parent and child features are more related and follow the \textit{reconstruction} condition.

In particular, we measure the proportion of identified parent-child pairs that satisfy both activation coverage and reconstruction condition, the more pairs the better following our argument in Section \ref{sec: reconstruction condition}. Specifically, for each candidate pair, we train a linear probe to obtain the true concept direction of the child feature, and check whether both parent and child decoder vectors are among the top-5 features most correlated with this direction. For each SAE, we sample 100 parent features from the denser half of the feature set (as denser features are more likely to have children), and use either MCS or Tree Structure to identify child features per parent. For Tree SAE, which supports both procedures, we sample an equal number of children from each (e.g., if a parent has 7 children via Tree Structure, we also retrieve the top-7 via MCS).

% \begin{figure}[ht]
%   \begin{center}
%     \centerline{\includegraphics[width=1\columnwidth]{figures/compare_tree.png}}
%     \caption{The result of the Hierarchy metric of the SAEs across all $L_0$ levels. The averages are marked on the right of the plot. 
%     % Tree SAE with Tree Structure and MCS substantially outperform other SAEs. Furthermore, the Tree Structure is slightly better at detecting hierarchical pairs than MCS.
%     } 
%     \vspace{-3mm}
%     \label{fig: compare tree}
%   \end{center}
% \end{figure}

% \begin{figure}[ht]
%   \begin{center}
%     \vspace{-5mm}
%     \centerline{\includegraphics[width=1\columnwidth]{figures/extended_compare_tree.png}}
%     \caption{The result of the Hierarchy metric of the SAEs across all sizes. The averages are marked on the right of the plot. 
%     % Tree SAE with Tree Structure and MCS outperforms Matryoshka SAEs across all sizes. Moreover, the same trend is shown where Tree Structure is slightly more accurate at detecting hierarchical feature pairs.
%     }
%     \label{fig: scale compare tree}
%     \vspace{-8mm}
%   \end{center}
% \end{figure}

The results in Figure \ref{fig: compare tree} show that Tree SAEs discover substantially more hierarchical pairs than others, for both tree structure and MCS. Notably, tree structure is slightly more effective at identifying hierarchical pairs than MCS. On the other hand, the feature pairs identified by MCS are often less coherent, representing unrelated concepts.

\subsection{Reconstruction Loss and Interpretability}
\label{exp: reconstruction loss}
\textbf{Reconstruction Loss.} We evaluate the quality of the reconstruction of the SAEs by quantifying the downstream reconstruction loss and variance explained. Experiment details and results are in Appendix \ref{sec: reconstruction setup} and Figure \ref{fig: reconstruction}, respectively.
The result indicates that Tree SAE is not only competitive in hierarchical-related metrics, but also closely reconstructs the original latents, showing a substantial gap in reconstruction quality compared to both Matryoshka and Top-$k$ SAE. Furthermore, in contrast to other metrics, MP-SAE provides the best reconstruction quality, outperforming others.

\begin{figure}[ht]
  \begin{center}
    \vspace{-1mm}
    \centerline{\includegraphics[width=1\columnwidth]{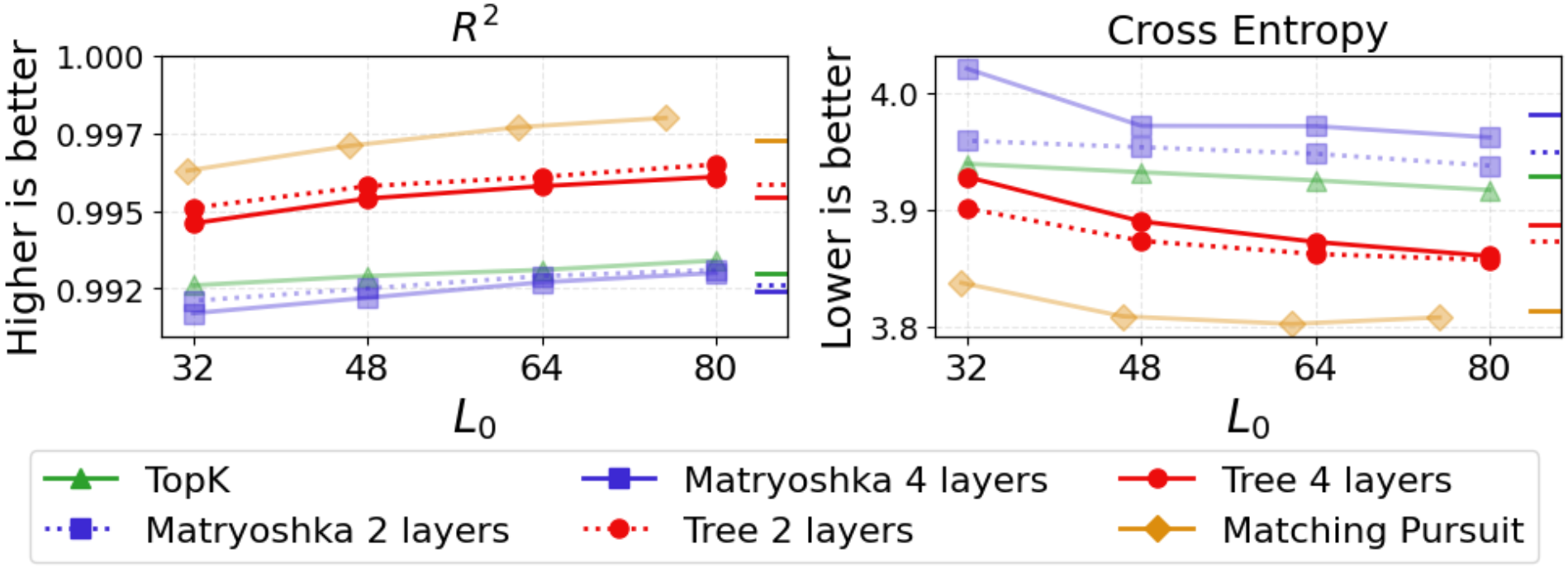}}
    \caption{Variance explained and Downstream loss of the evaluated SAE in four levels of sparsity. 
    } 
    \vspace{-5mm}
    \label{fig: reconstruction}
  \end{center}
\end{figure}

\begin{figure}[ht]
  \begin{center}
    \vspace{-3mm}
    \centerline{\includegraphics[width=1\columnwidth]{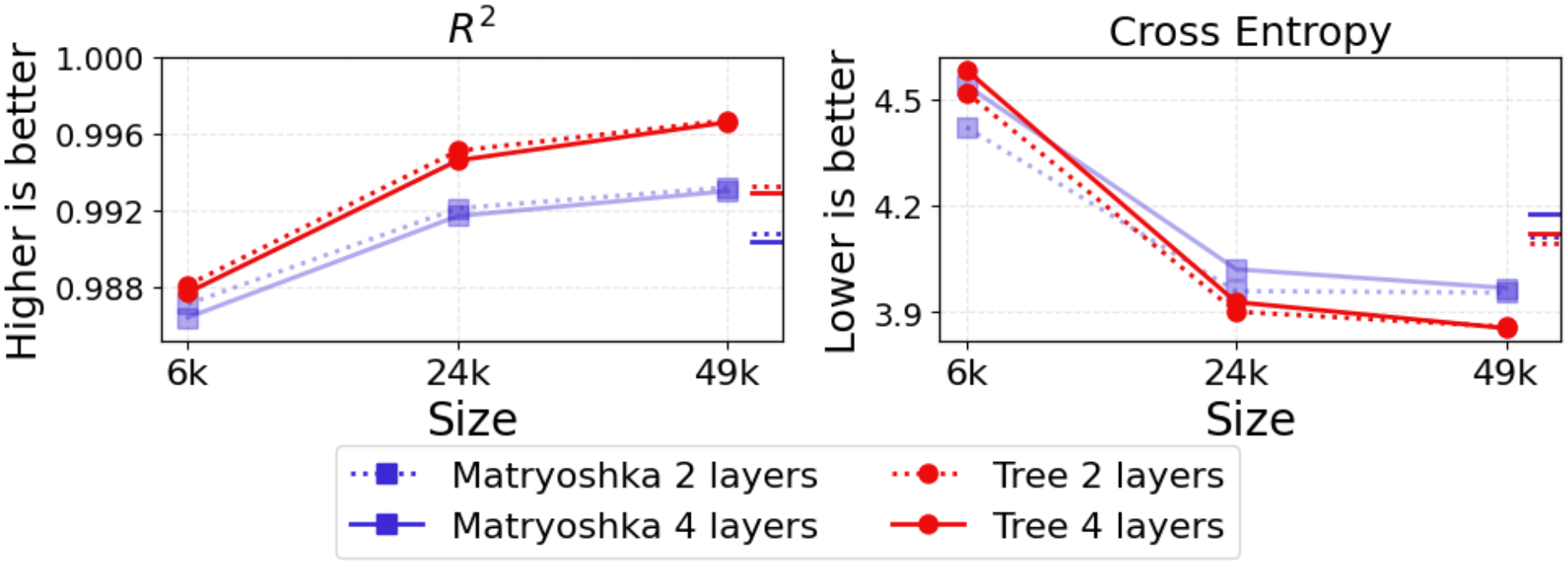}}
    \caption{Variance explained and Downstream loss of Tree SAE and Matryoshka at different scaling. 
    }
    \label{fig: scale reconstruction}
    \vspace{-8mm}
  \end{center}
\end{figure}

% \begin{figure*}[tbh]
%     \centering
%     \vspace{-10mm}
%     \begin{minipage}{.49\textwidth}
%     \centering
%         \includegraphics[width=0.9\columnwidth]{figures/extended_reconstruction_loss.png}
%         \vspace{-1mm}
%         \caption{Variance explained and Downstream loss of Tree SAE and Matryoshka at different scaling. 
%         }
%     \label{fig: reconstruction}
%     \end{minipage}
%     \hfill
%     \begin{minipage}{.48\textwidth}
%         \centering
%         \includegraphics[width=0.9\columnwidth]{figures/reconstruction_loss.png}
%         \vspace{-1mm}
%         \caption{Variance explained and Downstream loss of the evaluated SAE in four levels of sparsity. 
%         } 
%     \label{fig: scale reconstruction}
%     \end{minipage}
% \end{figure*}

\begin{figure*}[h]
\begin{center}
\vspace{-8mm}
\includegraphics[width=0.9\textwidth]{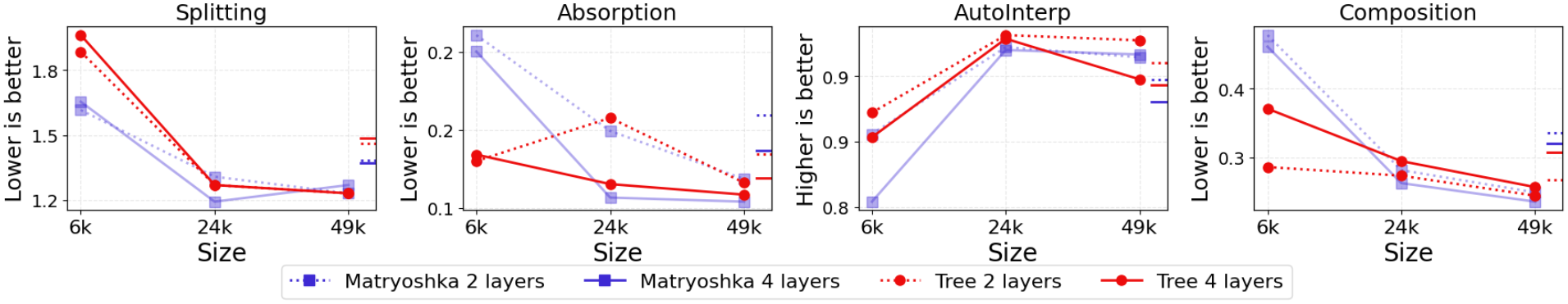} 
\end{center}
\vspace{-2mm}
\caption{Comparison between Matryoshka and Tree SAE at different dictionary sizes on Feature Splitting, Absorption, AutoInterp, and Composition. The averages across all $L_0$ levels are marked on the left of the plots. 
} 
\label{fig: extend_hierarchy}
\end{figure*}

\begin{figure*}[h]
\begin{center}
\vspace{-1mm}
\includegraphics[width=0.7\textwidth]{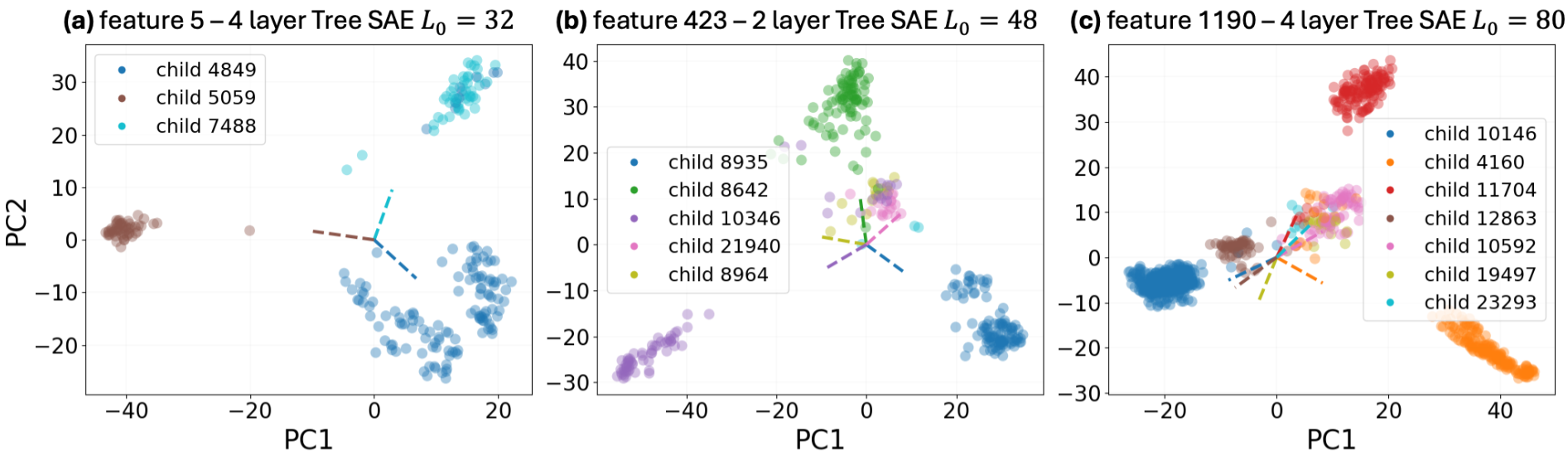} 
\end{center}
\vspace{-3mm}
\caption{Visualisation of the activation child features of the parent features using PCA projection. The corresponding feature encoder vectors are projected onto the same space. The figure indicates that the learned feature vectors of Tree SAE can correctly identify the child feature subspace, allowing hierarchical concept geometry analysis.} 
\vspace{-4mm}
\label{fig: geometry}
\end{figure*}

% \subsection{Interpretability}
% \label{exp: interp}
\textbf{Interpretability.} We follow the procedure of AutoInterp \cite{auto_interp, automatically_interp_millions} in measuring interpretability of the features, with more details in Appendix \ref{sec: autointerp setup}. Our results in Figure \ref{fig: main_hierarchy} show that the Tree SAEs are highly interpretable, matching the performance of Matryoshka for both 2 and 4 layers version. Top-$k$ SAE remains competitive scores, it has a slightly worse AutoInterp score at $L_0 = 80$. In contrast, MP SAE has less interpretable features than others; its score lowers as the density increase.

\subsection{Scaling of Tree SAEs}
\label{exp: scalling}
We compare the scaling of our SAE with the previous SOTA Matryoshka SAE on the hierarchy-related metrics. We find that Tree SAE scales better than Matryoshka (Figure \ref{fig: extend_hierarchy} and \ref{fig: scale compare tree}), especially at lower dictionary sizes, details in Appendix \ref{exp: apd scalling}.

\section{Usefulness of Tree Structure}
\label{exp: usefulness}
We provide qualitative and quantitative experiments to show the usefulness of leveraging an explicit tree structure that accurately identifies hierarchical pairs. 
We show that using tree structure allows the analysis of hierarchical concept geometry in Section \ref{exp: geometry}; and can help understand how a language model breaks down concepts into coarse- and fine-grained levels in Section \ref{exp: interpret tree structure}. We then evaluate the diversity of the child features in Section \ref{exp: diversity} to validate that the child features learn to represent different aspects of the parent features.

\subsection{Geometry of Child Feature Space}
\label{exp: geometry}

We provide evidence that Tree SAE can be used to study the geometry of hierarchical concepts. Specifically, we study the geometry of the child concept subspace given a parent. 
% For a random parent feature 26 from Tree SAE 4 layers $L_0 = 32$, we first sample the hidden activation of GPT2-small, where all of the child features activate. Following \cite{saegeometry}, we use LDA to separate the activation of child features and extract the top 2 PCA directions to visualise
For a parent feature, we first sample the hidden activation of GPT2-small, where all of the child features activate, then we use PCA to visualize on the top-2 PCA components.
On the same subspace, we also project the feature encoder vectors of the corresponding child features. The results in Figure \ref{fig: geometry} show that the feature vectors correctly point toward the child concepts on the subspace, suggesting that the Tree SAE encoder correctly learns the subspace that separates the child concepts, allowing geometric analysis of hierarchical concepts. Note that, although the first two examples' features (sub-figures a and b) correctly point toward the child activation clusters, we do find that in some cases, such as sub-figure c, the feature vector of 19497 points toward the opposite direction.  However, in such cases, we find that this is often due to either the inability to separate the child concepts of PCA or the overlapping activation between the child features, and therefore projecting on a noisy subspace leads to poor results.

\subsection{Interpret Features Using Tree Structure}
\label{exp: interpret tree structure}
We demonstrate some qualitative observations on the tree structure of our Tree SAE, showing that it can be used to explore the decomposition of general concepts to finer-grained concepts. Figure \ref{fig: interpret tree structure1},  \ref{fig: interpret tree structure2}, \ref{fig: interpret tree structure3}, \ref{fig: interpret tree structure4} show the learned structures from parent features. We find that while the parent features have more general meanings, the child features break down the activations into interpretable and more specialized cases. For example, in Figure \ref{fig: interpret tree structure2}, the tree structure correctly classifies related concepts as it separates company names from the general feature and further breaks down concepts across coarse‑ and fine‑grained scales. This, qualitatively, shows the interpretability of our tree structure.

% \begin{figure}[ht]
%   \begin{center}
%     \centerline{\includegraphics[width=1\columnwidth]{figures/interpret_tree_structure1.png}}
%     \caption{Qualitative observation on the tree structure of root feature 3410 of 2 layers Tree SAE at $L_0=48$. 
%     % The features in the tree break down the general polysemantic concept into monosemantic concepts.
%     }
%     \label{fig: interpret tree structure1} 
%   \end{center}
% \end{figure}

% \begin{figure}[ht]
%   \begin{center}
%     \centerline{\includegraphics[width=1\columnwidth]{figures/interpret_tree_structure2.png} }
%     \caption{Qualitative observation on the tree structure of root feature 135 of 4 layers Tree SAE at $L_0=80$. 
%     % The features in the tree break down the general polysemantic concept into monosemantic concepts.
%     }
%     \vspace{-4mm}
%     \label{fig: interpret tree structure2}
%   \end{center}
% \end{figure}

\subsection{Child Features Diversity}
\label{exp: diversity}
For Tree SAEs to be useful in exploring hierarchical structure, we also want the child features to have non-overlap activation. If the child features from the same parent have a high co-activation rate, then they represent the same concept, suggesting a poor diversity in the feature hierarchy. We therefore measure the average co-activation rate for all of the features from the same parent. The results are shown in table \ref{tab: co-occurrence}. We find that both two-layer and four-layer Tree SAEs have a low co-occurrence rate, only less than 3.0\%. This shows that the learned child features activate on different cases of the parent feature, suggesting a high diversity in the feature set.

%% file: sections/5-discussion.tex
In this paper, we showed that prior work has important limitations in defining and detecting hierarchical features. We proposed a stronger definition of hierarchical feature pairs and introduced a novel SAE that better captures hierarchical structure and mitigates hierarchy-related pathologies. Our SAE learned substantially more parent–child pairs while remaining competitive on other benchmarks. 
% We believed that Tree SAE is a promising step toward uncovering underlying hierarchical structure and enabling geometric analysis and interpretation of how models encode concepts.

\section{Limitations}
\label{sec: limitation}
Although attaining promising results, our SAE contains a high number of dead features compared to other SAEs due to the rigid structure. While dynamic allocation helps, we believe that this problem can not be solved by proposing stronger allocation optimisation, but instead, employing soft-gating on the tree structure might be a promising direction. Furthermore, although learning more hierarchical pairs, Tree SAEs' structure still contains pairs that violate the reconstruction condition. Further improvement might incorporate the condition more gracefully during the training of the SAE to avoid this problem. We leave these directions for future work.

%% file: sections/appendix.tex
\begin{figure*}[tbh]
    \centering
    \begin{minipage}{.48\textwidth}
        \includegraphics[width=1\columnwidth]{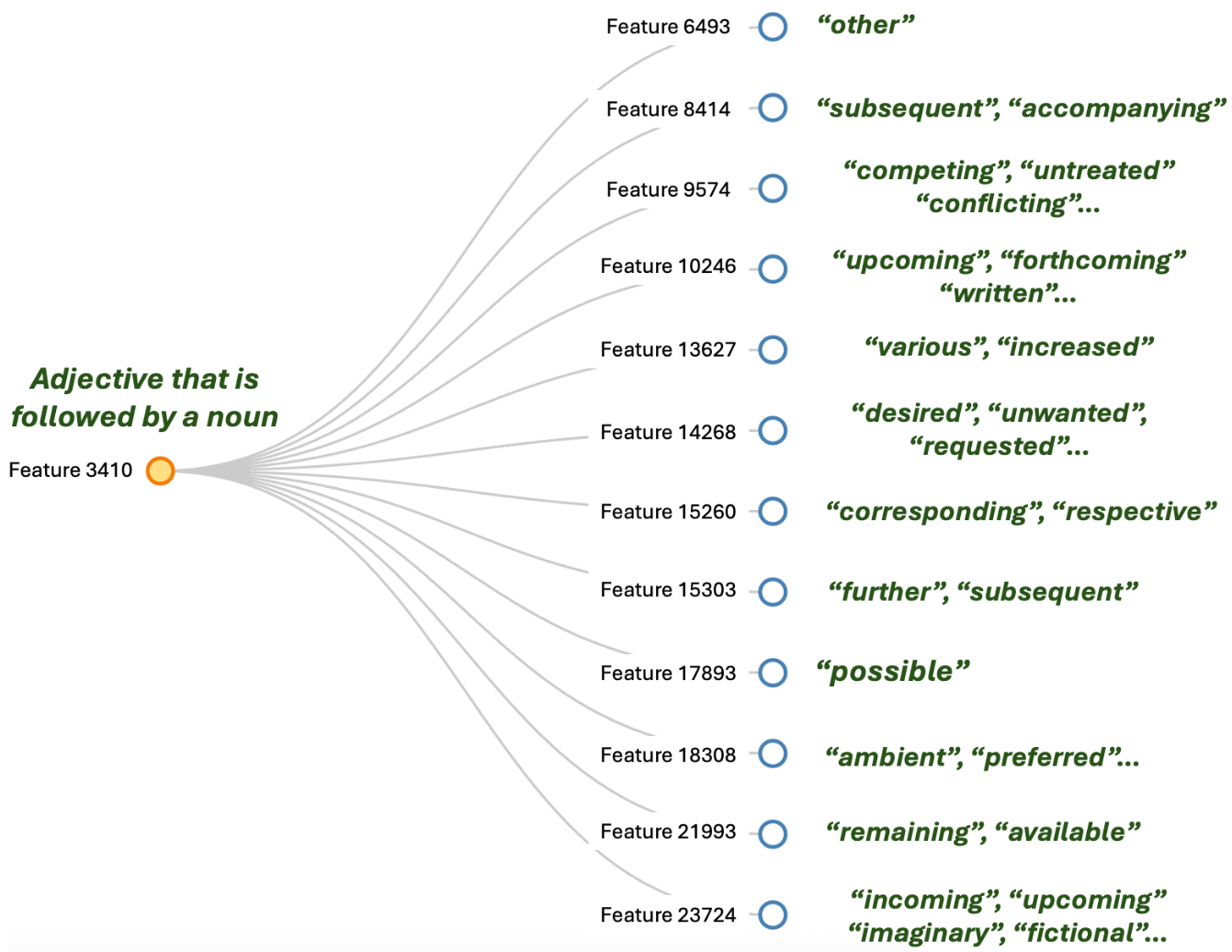}
        \vspace{-4mm}
        \caption{Qualitative observation on the tree structure of root feature 3410 of 2-layer Tree SAE at $L_0=48$. 
        % The features in the tree break down the general polysemantic concept into monosemantic concepts.
        }
    \label{fig: interpret tree structure1}   
    \end{minipage}
    \hfill
    \begin{minipage}{.50\textwidth}
        \vspace{3mm}
        \includegraphics[width=1\columnwidth]{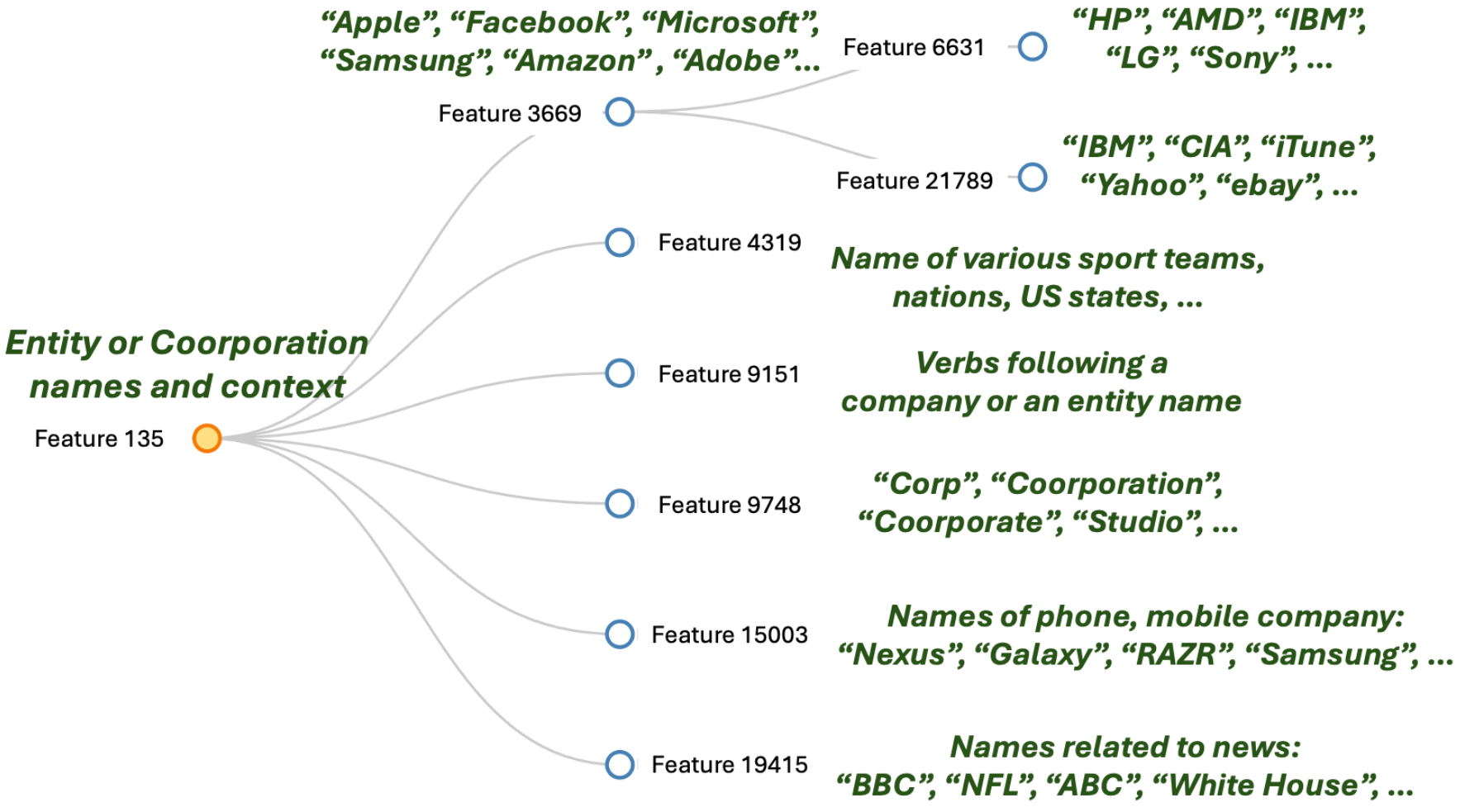} 
        \vspace{7mm}
        \caption{Qualitative observation on the tree structure of root feature 135 of 4 layers Tree SAE at $L_0=80$. 
        % The features in the tree break down the general polysemantic concept into monosemantic concepts.
        }
    \label{fig: interpret tree structure2}
    \end{minipage}
\end{figure*}

\begin{figure*}[tbh]
    \centering
    \begin{minipage}{.53\textwidth}
        \includegraphics[width=1\columnwidth]{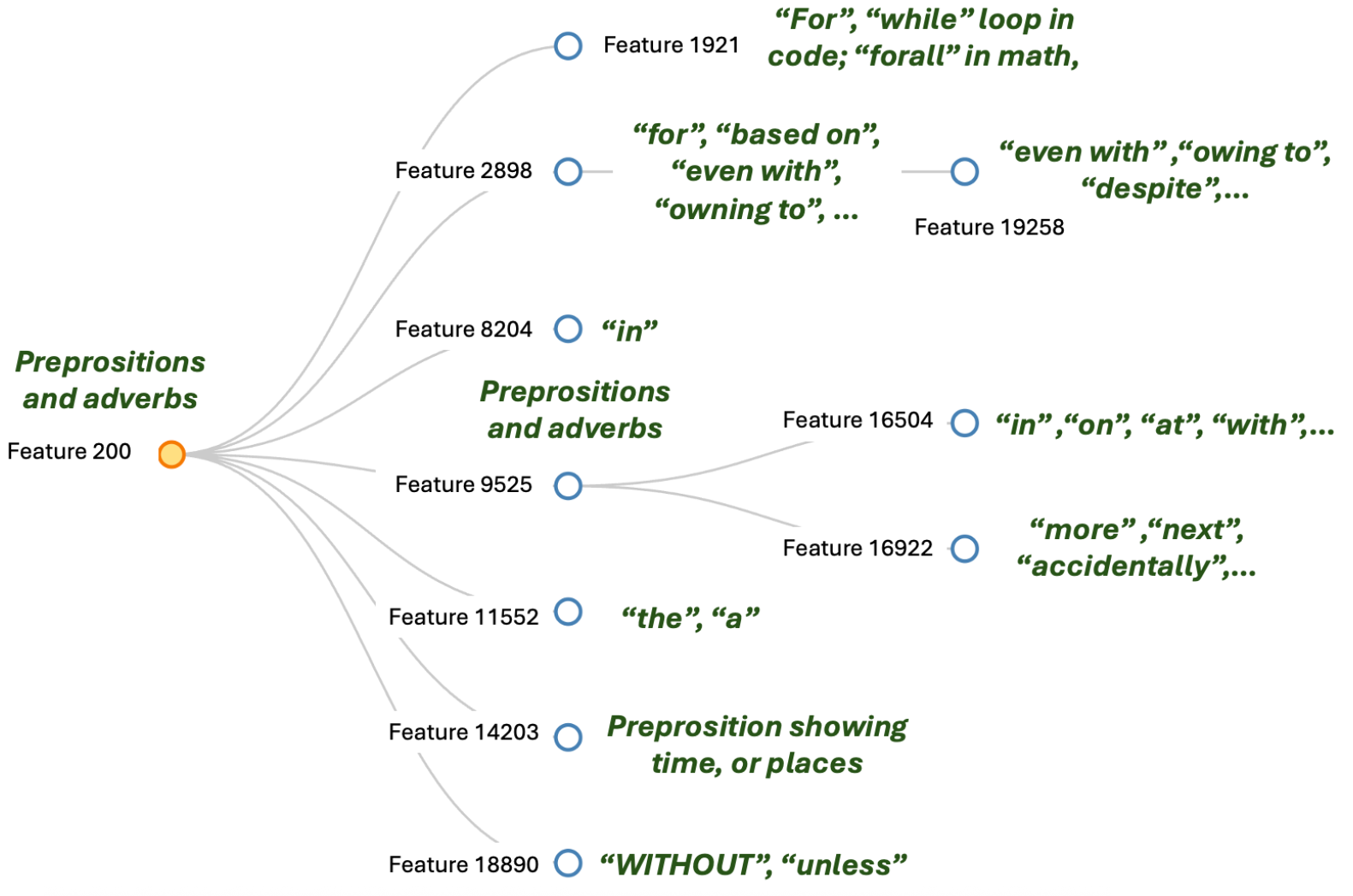}
        \vspace{3mm}
        \caption{Qualitative observation on the tree structure of root feature 200 of 4 layers Tree SAE at $L_0=64$.}
    \label{fig: interpret tree structure3}   
    \end{minipage}
    \hfill
    \begin{minipage}{.45\textwidth}
        \includegraphics[width=1\columnwidth]{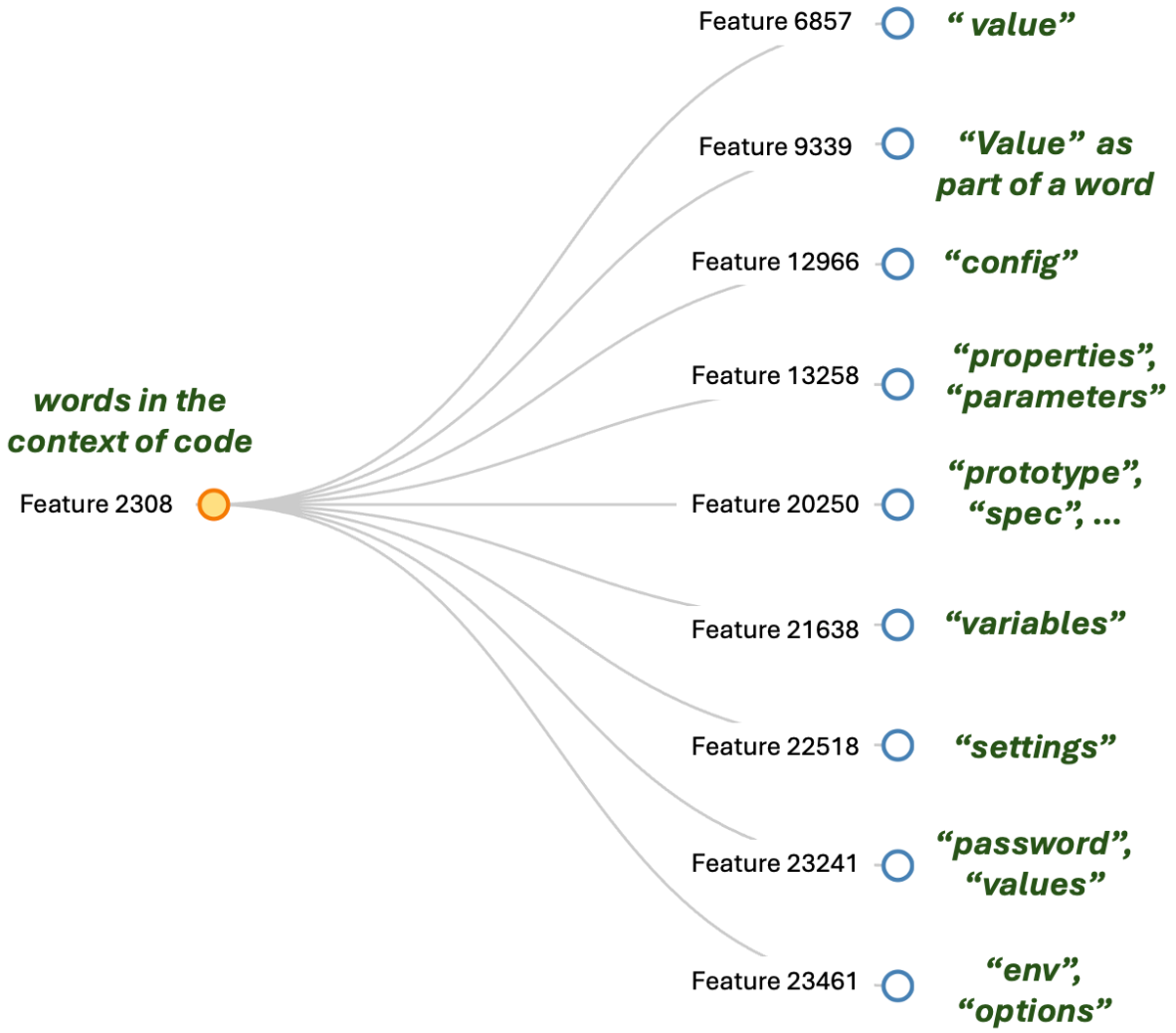} 
        \vspace{-3mm}
        \caption{Qualitative observation on the tree structure of root feature 2308 of 2 layers Tree SAE at $L_0=32$.}
    \label{fig: interpret tree structure4}
    \end{minipage}
\end{figure*}

\section{Related Works}
\label{sec: related works}
\subsection{Sparse Autoencoder}
Sparse Autoencoder (SAE) has been considered as the standard way to understand the hidden latent of language model \cite{monosemanticity, jumprelu, cunningham2023sparse, lieberum2024gemma, batchtopk}. It breaks down the latent into a sparse set of interpretable features where each feature selectively activates on a specific context or idea. Inspecting the extracted features from SAEs has led to various applications, including model steering \cite{sae_steering}, model diffing \cite{diffing}, or even tracing the ``thought" of LLM by finding the ``circuit" inside the model \cite{feature_circuit, transcoder_circuit}. The formalisation of the standard SAE architecture, as well as the training objective, is provided in Section \ref{sec: preliminary}.

\subsection{Challenges in Sparsity Training and Learning Hierarchical Features}
\label{sec: hierarchy-related pathology}
The key to the success of SAE is extracting a sparse set of features from an overcomplete dictionary \cite{monosemanticity}, where the features are often human-interpretable and useful in understanding the model \cite{cunningham2023sparse}. However, optimising for reconstruction and sparsity often leads to undesirable phenomena:

\textbf{Feature Splitting and Absorption:} Feature Splitting occurs when a set of features detects the specialised and fragmented concepts, such as "punctuation mark and comma" and "punctuation mark and question", without learning a generalised "punctuation marks"  concept \cite{monosemanticity, absorption}. More problematically, providing more capability for the SAE (scaling the SAE) will make the splitting issue worse \cite{monosemanticity}. This pathology leads to a set of incomplete features that simultaneously represent fragmented and not generalised concepts. On the other hand, Feature Absorption \cite{absorption} happens when the feature representing a concept developed systematic ``blind spots" because of more specialised features. For example, consider a general feature that activates on female names like ``Mary", ``Lily", ``Jane", etc. If another feature specialises in detecting only ``Lily", the sparsity optimisation will likely push the general feature to activate on ``female name except for Lily", leaving a ``hole" in the activation of the general feature \cite{matryoshka}. 

\textbf{Feature Composition:} Feature Composition is a problem where the feature represents a composition concept rather than representing using underlying independent features (i.e. ``red triangle" feature instead of ``red" and ``triangle" features) \cite{composition, metasae}. While combining independent features into a specialised and compositional feature can have the same reconstruction loss and lower $L_0$, this prevents SAE to learn an atomic and unique feature set, posing an urge to learn hierarchical features that correctly decompose, generalise and specialise concepts.

\subsection{Hierarchy SAE}
\label{sec: hierarchy sae}
In the attempt of learning hierarchical features, several ideas have been proposed. \cite{matryoshka} propose Matryoshka SAE that learns multi-level features using Matryoshka representation \cite{matryoshka_representation}. Specifically, they encourage the features to reconstruct the input at multiple layers, fostering the learning of both generalised and specialised features. This has been shown \cite{matryoshka} to mitigate all three mentioned pathologies, namely Feature Splitting, Absorption, and Composition. Furthermore, \cite{matryoshka} outperforms standard SAE \cite{batchtopk} in a toy dataset experiment with known hierarchical features; Matryoshka correctly learns the parent and child features without suffering from absorption. \cite{mp} proposes Matching Pursuit SAE that applies the well-known Matching Pursuit algorithm to learn hierarchical features. The features set is ``conditional orthogonal" \cite{mp, hierarchical_geometry} where parent and child features span in orthogonal subspace. This SAE is shown to faithfully learn hierarchical structure on a toy dataset with both inter- and intra-level feature correlation, while baseline SAE \cite{batchtopk} and Matryoshka SAE \cite{matryoshka} cannot fully reconstruct. However, despite the initial success, all of the previous hierarchical SAEs learn independent features and can not point out the relations between the feature pairs. Furthermore, even when an algorithm such as MCS \cite{monosemanticity, matryoshka} is used to detect the structure, the identified pairs are not guaranteed to be hierarchical, as we show in our experiments and evidence.

\section{Scaling of Tree SAEs}
\label{exp: apd scalling}
We compare the scaling of our Tree SAE with the previous state-of-the-art Matryoshka SAE on the hierarchy-related metrics. We scale the dictionary size to 6k and 49k for both of the SAEs with the same training hyperparameters. We test the SAEs on the smallest $L_0$ setup and scale proportionally to the dictionary size, as it is the hardest setup for Feature Splitting, Absorption, and Composition. The training setup are provided in Appendix \ref{sec: sae training}. 

We find that Tree SAE scales well with different dictionary sizes compared to Matryoshka SAE (Figure \ref{fig: extend_hierarchy} and \ref{fig: scale compare tree}). Specifically, the Tree SAEs outperform previous state-of-the-art SAEs at Feature Splitting and Feature Absorption at small dictionary size, while having relatively the same performance on large sizes, resulting in better scores on average. Surprisingly, we observe the same trend for Feature Composition, where the Tree SAE significantly outperforms Matryoshka SAE at lower dictionary size and has a comparable Composition rate for the remaining setups. For the AutoInterp metric, we find that Tree SAE is again having a stronger score at 6k; however, as the size increases, both SAEs have roughly the same performance, with the 4-layer Tree SAE at 49k slightly degrading. On the Hierarchy metric, we find a similar trend as in Section \ref{exp: hierarchy} that Tree SAE consistently identifies more hierarchical pairs than the counterpart, and our tree structure is more accurate at finding hierarchical features than MCS. Our results suggest that Tree SAE can scale well compared to the SOTA on hierarchy-related metrics.

\section{Dynamic Allocation Details}
\label{sec: dynamic allocation detail}
We provide proof that solving Equation (\ref{eq: dynamic alloc}) has the complexity of $O(s_l\log(m_l))$ using a greedy algorithm where $s_l$ and $m_l = \sum_{t=0}^{l-1}s_t$ are the number of features and parents at privilege layer $l$. For each parent, let $k_{l, p} \geq 0\:, \: p\in \{1, \dots m_l\}$ be the number child features from the allocation $\mathbf{a}_l$:

\textbf{Theorem 4.1:} For any $\tau > 0$, there exists an allocation $\mathbf{a}_l$ with $k_{l, 1}, \dotsc, k_{l, m_l} \in \mathbb{Z}_{\geq 0}$, $\sum_p k_{l,p} = s_l$, such that $\min_{p| k_{l, p} > 0} (C_p / k_{l, p}) \geq \tau$ if and only if
\[
\sum_{p=1}^{m_l} \left\lfloor \frac{C_p}{\tau} \right\rfloor \geq s_l.
\]

\textbf{Proof:}

\noindent \textbf{Only if.}
Assume an allocation with $\sum_{p=1}^{m_l} k_{l, p} = s_l$ achieves $\min_{p| k_{l, p} > 0} (C_p / k_{l, p}) \geq \tau$. Then for every $p$ with $k_{l, p} > 0$,
\[
\frac{C_p}{k_{l, p}} \geq \tau \Rightarrow k_{l, p} \leq \frac{C_p}{\tau} \Rightarrow k_{l, p} \leq \left\lfloor \frac{C_p}{\tau} \right\rfloor.
\]
For $p$ with $k_{l, p} = 0$, the inequality also holds. Summing over $p$:
\[
s_l = \sum_{p=1}^{m_l} k_{l, p} \leq \sum \left\lfloor \frac{C_p}{\tau} \right\rfloor.
\]
Hence $\sum \left\lfloor \frac{C_p}{\tau} \right\rfloor \geq s_l$.

\noindent \textbf{If.}
Assume $\sum \left\lfloor \frac{C_p}{\tau} \right\rfloor \geq s_l$. Pick any integers $k_{l, p}$ with $0 \leq k_{l, p} \leq \left\lfloor \frac{C_p}{\tau} \right\rfloor$ and $\sum_{p=1}^{m_l} k_{l, p} = s_l$ (e.g., greedily fill parents up to their caps $\left\lfloor \frac{C_p}{\tau} \right\rfloor$ until we place $s_l$ children; this is always possible because the sum of caps is at least $s_l$). For any $p$ with $k_{l, p} > 0$,
\[
k_{l, p} \leq \left\lfloor \frac{C_p}{\tau} \right\rfloor \leq \frac{C_p}{\tau} \Rightarrow \frac{C_p}{k_{l, p}} \geq \tau.
\]
Therefore $\min_{p| k_{l, p} > 0} (C_p / k_{l, p}) \geq \tau$. This completes the proof for the theorem.

\textbf{Optimality of Greedy:}

Define the optimal value:
\[
\tau^* = \sup \left\{ \tau > 0 : \sum \left\lfloor \frac{C_p}{\tau} \right\rfloor \geq s_l \right\}.
\]

Equivalently, let $S$ be the multi-set:
\[
S_{s_l} = \left\{ \frac{C_p}{k} : p = 1,\dotsc,m_l;\ k = 1,2,3,\dotsc \right\}.
\]

Then $\tau^*$ equals the $s_l$-th largest element of $S$ (ties allowed). This is because: assuming that the $s_l$ largest elements are $S_{s_l}=\{\frac{C_p}{1}, \dotsc,\frac{C_p}{k'_p}: p = 1,\dotsc,m_l; \, k'_p \geq 1\}$ (since the larger $k'_p$, the smaller the element), then $\sum_{p| \: k'_p \geq 1} k'_p = s_l$. Let $\tau^*$ be the smallest element in $S_{s_l}$, then $\sum \left\lfloor \frac{C_p}{\tau^*} \right\rfloor \geq s_l$. As stated in theorem \ref{theorem: threshold feasibility}, there exist an allocation satisfy the problem in Equation (\ref{eq: dynamic alloc}) for optimal value $\tau^*$.

Therefore, at each step, we can greedily allocate one child to the parent with the highest payoff; we can implement this efficiently using a heap with the complexity of $O(\log(m_l))$. The full greedy algorithm is provided in Algorithm \ref{alg: dynamic alloc}. We allow additional eligible conditions for a feature to be a parent, avoiding too sparsely activated parent features. Specifically, we require that the parent features have an activation rate of 1 every 50,000 tokens or higher. This activation rate is based on our observation that sparser features should not be decomposed further. 

\begin{algorithm}[tbh]
    \caption{Greedy allocation algorithm.}
    \textbf{Input}: Capacity set $\mathbb{C}_l=\{C_p, \: p \in \{1, 2, \dots, m_l\}\}$; Number of child features $s_l$ at layer $l$.
    \begin{algorithmic}[0]
        \STATE Initialize the (optimal) number of child per parent $k^*_{l, p}$
        \STATE $H \leftarrow \emptyset$
        \STATE $m_l \leftarrow |\mathbb{C}|$
        \FOR{parent $p$ in $\{1, \dots, m_l\}$}
            \STATE \COMMENT{See Section \ref{sec: dynamic allocation detail} for eligibility condition}
            \IF{$parent\_is\_eligible()$}  
                \STATE $H \leftarrow H \,\cup \, (C_p / 1, p)$ 
            \ENDIF
        \ENDFOR
        \STATE Heapify $H$
        \STATE $assigned \leftarrow 0$
        \WHILE{$H \neq \emptyset$ and $assigned < s_l$}
            \STATE $c, \, p \leftarrow H.heap\_pop()$
            \STATE $k^*_{l, p} \leftarrow k^*_{l, p} + 1$
            \STATE $c \leftarrow C_p / (k^*_{l, p}+1)$
            \STATE $H.heap\_push((c, p)$)
            \STATE $assigned \leftarrow assigned + 1$
        \ENDWHILE
    \end{algorithmic}\label{alg: dynamic alloc}
    \textbf{Return}: $\mathbf{k}^*_l$
\end{algorithm}

We provide the full reallocation algorithm in algorithm \ref{alg: full algo}. In particular, we compute the best numbers of children $\mathbf{k}^*_l$, then sample the pool of dead child features, usually defined as features that are inactive for 10M tokens \cite{topk, cunningham2023sparse}. In the step of assigning child features to match the optimal allocation, we employ first fit strategy for all of the parents that have less than the optimal number of non-dead child features. Optionally, we can enforce a prerequisited number of child features for the root node at $l=0$. We find this can helps reducing number of dead features in Gemma-2-2b \cite{gemma_2}.

\begin{algorithm}[tbh]
    \caption{Full dynamical feature reallocation algorithm.}
    \textbf{Input}: Current assignment vectors $\{\mathbf{a}_l \: \forall l \in \{1, \dots, L\}\}$; Number of child features $s_l$ at layer $l$.
    \begin{algorithmic}[0]
        \FOR{layer $l$ in $\{1, \dots, L\}$}
            \STATE Compute capacity set $\mathbb{C}_l=\{C_p, \: p \in \{1, 2, \dots, m_l\}\}$
            \STATE Compute the number of children for each parent with the current assignment $\mathbf{a}_l$: $\mathbf{k}_l$
            \STATE $\mathbf{k}^*_l \leftarrow$ $Greedy\_allocation\_algorithm(\mathbb{C}_l, s_l)$ \: \: \: \COMMENT{algorithm \ref{alg: dynamic alloc}}
            \STATE Sample dead child features pool at layer $l$
            \STATE $\mathbf{a}_l \leftarrow$ Assigning child features to match $\mathbf{k}^*_l$
        \ENDFOR
    \end{algorithmic}\label{alg: full algo}
    \textbf{Return}: $\{\mathbf{a}_l \: \forall l \in \{1, \dots, L\}\}$
\end{algorithm}

We compare the Tree SAE setup with and without dynamic dead-feature allocation. We run the experiment on a 4-layer Tree SAE with $L_0 = 32$ (the hardest setup to avoid dead features). To see the effect of auxiliary loss and dynamic allocation, we use auxiliary loss for all layers in both runs, while allowing dynamic allocation only in one run. Figure \ref{fig: dead feature rate} shows that by allocating child features to more needed parents, the Tree SAE with dynamic allocation achieves an almost zero dead-feature rate at all layers starting at 40k steps, indicating that the learned features receive many gradient updates throughout training. On the other hand, the Tree SAE without dynamic allocation has an 8\% dead-feature rate at layer 3, whereas at layer 2 the dead-feature rate reaches zero at the end of training, indicating that features receive fewer gradient updates and are poorly learned. Furthermore, as shown in Figure \ref{fig: dead feature rate}, relocating the features significantly reduces the number of dead features (indicated by the steep drop), demonstrating the effectiveness of our feature allocation method on multi-layer Tree SAE.

\begin{figure*}
\begin{center}
\includegraphics[width=1\textwidth]{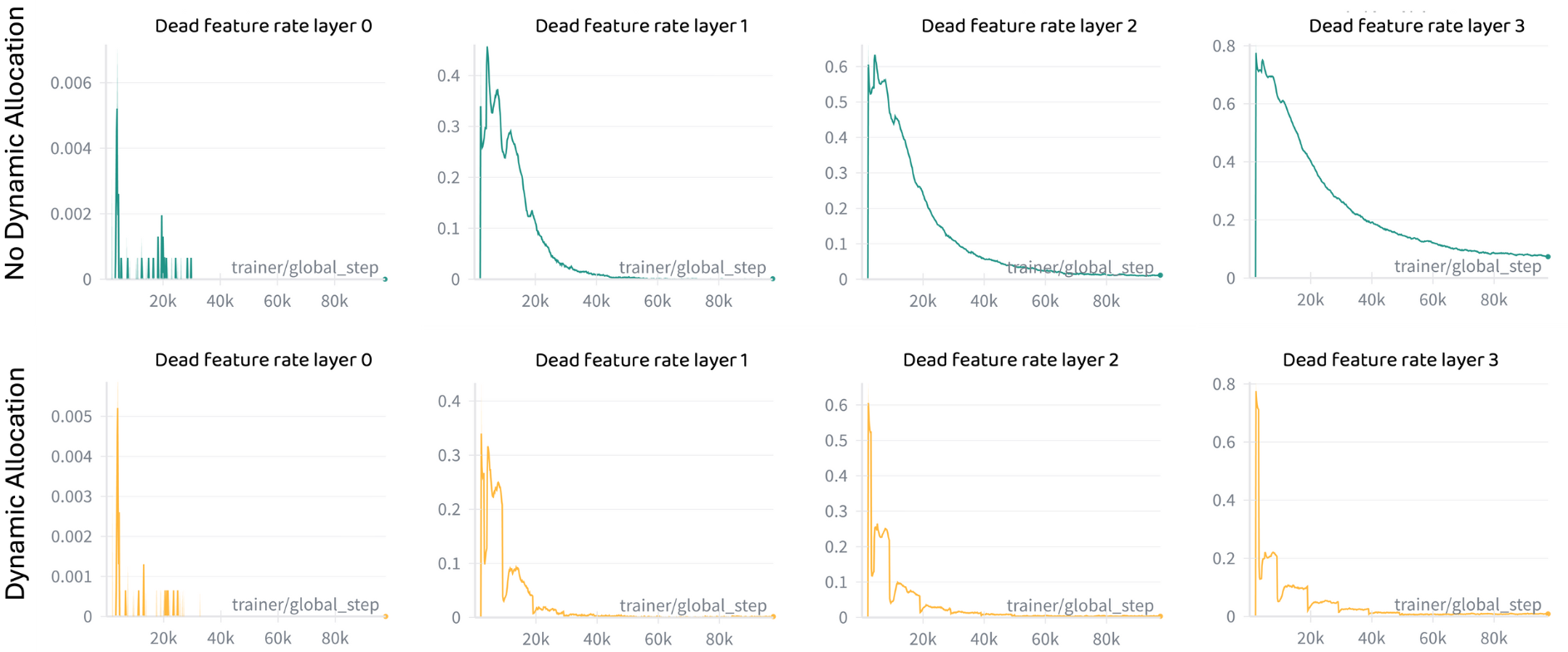} 
\end{center}
\caption{The comparison of the dead feature rate with and without the dynamic allocation method on a 4-layer Tree SAE with $L_0=32$.} 
\label{fig: dead feature rate}
\end{figure*}

\section{Multi-level Reconstruction Loss Combine with Activation Coverage Improve Reconstruction Score}
\label{sec: reconstruction loss and activation coverage improves reconstruction score}
In this section, we give answers to two key questions:
\begin{itemize}
    \item One might ask if the multi-level reconstruction loss (Equation \ref{eq: multilevel_reconstruction}) and tree structure (Equation \ref{eq: enforce activation coverage}) are necessary. We will show that these two components are crucial for the reconstruction score increase (Equation \ref{eq: reconstruction condition}). We therefore answer the following question: \textit{``How do multi-level reconstruction loss and tree Structure enforce the feature pairs to follow reconstruction condition?"}.

    \item One could argue to use child encoder as the approximation of ``true direction of child concept $\mathbf{d}_c^*$" instead of training a separated probe as we do in Section \ref{sec: activation coverage fail}. The answer is that the child features encoder and parent decoder vector have low cosine similarity, and thus, cannot differentiate between the correct and incorrect pairs. Consequently, we answer the following question: \textit{``Why child encoder has low cosine similarity with parent feature decoder?"}.
\end{itemize}

\textbf{Notations:} We consider the learning of two feature in our Tree SAE: parent feature $f_p\geq0$ and child feature $f_c\geq0$ to represent true concept vector be $\mathbf{d}_c^\ast$ with positive true activation $f_c^\ast>0$. We assume $f_p$ and $f_c$ are parent and child pair defined by our tree SAE structure (Equation \ref{eq: enforce activation coverage}), which strictly follow the activation coverage condition: $S_{cov}=1$ (Equation \ref{eq: activation coverage}). And let $S_p=\mathbf{d}_p^T\mathbf{d}_c^\ast$, $S_c=\mathbf{d}_c^T\mathbf{d}_c^\ast$, and $k=\mathbf{d}_p^T\mathbf{d}_c$. Without loss of generality, we:

\begin{itemize}
    \item constrain all vectors to the unit sphere: $||\mathbf{d}_p||_2=||\mathbf{d}_c||_2=||\mathbf{d}_c^\ast||_2=1$, where we normalize the decoder vectors at after each gradient update.
    
    \item rescale the activation: $f_p/f_c^\ast=\alpha,f_c/f_c^\ast=\beta$
    
    \item consider a standard ReLU activation SAE: $f_p=(\mathbf{e}_p^T\mathbf{d}_c^\ast)f_c^\ast$ and $f_c=(\mathbf{e}_c^T\mathbf{d}_c^\ast)f_c^\ast$, where $\mathbf{e}_p, \mathbf{e}_c$ are the encoder vector of $f_p, f_c$ respectively.
\end{itemize}

We have $\alpha, \beta, \mathbf{d}_p, \mathbf{d}_c$ are independent.

We analyze two loss functions, our loss (Equation \ref{eq: multilevel_reconstruction}): $L_1=||\alpha \mathbf{d}_p-\mathbf{d}_c^\ast||^2_2+||\alpha \mathbf{d}_p+\beta \mathbf{d}_c-\mathbf{d}_c^\ast||^2_2$, and standard SAE loss: $L_2=||\alpha \mathbf{d}_p+\beta \mathbf{d}_c-\mathbf{d}_c^\ast||^2_2$.

\textbf{Q1: How do multi-level reconstruction loss and tree Structure enforce the feature pairs to follow reconstruction condition?:}

\textbf{Proporsition 1:} We will prove that optimize using $L_1$ will either yield high reconstruction score $S_{res}=\min(S_p,S_c)$ or $f_p$ learns the concept directly ($S_p=1, S_c=0$), while using $L_2$ can suffer from $S_c=1, S_p=0$ where the parent feature are unrelated to the child concept (similar to example in Figure 1).
\label{proposition1}

\textbf{Proof:} We rewrite $L_1$ as: 
\begin{multline}
    L_1(\alpha, \beta,\mathbf{d}_p,\mathbf{d}_c)=2\alpha^2\mathbf{d}_p^T\mathbf{d}_p+\beta^2\mathbf{d}_c^T\mathbf{d}_c+2 \\
    -4\alpha S_p-2\beta S_c+2\alpha\beta k
    \label{eq: l1_expansion}
\end{multline}

The gradients w.r.t each feature vector are: 
\[\partial L_1/\partial \mathbf{d}_p=-4\alpha \mathbf{d}_c^\ast+2\alpha\beta \mathbf{d}_c+4\alpha^2\mathbf{d}_p,\] 
\[\partial L_1/\partial \mathbf{d}_c=-2\beta \mathbf{d}_c^\ast+2\alpha\beta \mathbf{d}_p+\beta^2\mathbf{d}_c.\] 
Because of the term $2\alpha\beta \mathbf{d}_c$ in $\partial L_1/\partial \mathbf{d}_p$, and $2\alpha\beta \mathbf{d}_p$ in $\partial L_1/\partial \mathbf{d}_c$; the decoder vector of the two features will be pushed near orthogonal to each other, i.e. $|k|=|\mathbf{d}_p^T\mathbf{d}_c|\ll1$ (we also empirically justify that $|k|\ll1$ with an experiment below Table \ref{tab:measure k}). Similar argument, both $\mathbf{d}_p$ and $\mathbf{d}_c$ will be pushed toward $\mathbf{d}_c^\ast$, hence, we expect $S_p, S_c\in[0,1]$ at convergence.

At convergence, $\partial L_1/\partial \alpha=4\alpha-4S_p+2\beta k=0$ and $\partial L_1/\partial \beta=2\beta-2S_c+2\alpha k=0$. Solving these two equations yields the value of $\alpha$ and $\beta$ at convergence (here, we simplify $k^2\simeq0$ because $|k|\ll1$): $$\alpha^\ast=\frac{2S_p-kS_c}{2-k^2}\simeq S_p-kS_c/2,$$
$$\beta^\ast\simeq S_c-kS_p.$$ 

We rewrite Equation \ref{eq: l1_expansion} during convergence with $\alpha^\ast$ and $\beta^\ast$ and remove the term $k^2$: 
\begin{equation}
    L_1(\alpha^\ast,\beta^\ast,\mathbf{d}_p,\mathbf{d}_c)\simeq2-2S_p^2-S_c^2+2S_pS_ck.
    \label{eq: l1_approx}
\end{equation}

With similar derivation, we also obtain the form for $L_2$: 
\begin{equation}
    L_2(\alpha^\ast,\beta^\ast,\mathbf{d}_p,\mathbf{d}_c)\simeq1-S_p^2-S_c^2+2S_pS_ck.
    \label{eq: l2_approx}
\end{equation}

We now compare the landscape of $L_1$ (Equation \ref{eq: l1_approx}) and $L_2$ (Equation \ref{eq: l2_approx}) when $S_p, S_c \in [0, 1]$ when $k$ is small ($k=0.05$): Figure \ref{fig:loss_landscape}.

\begin{figure*}
    \centering
    \includegraphics[width=0.9\linewidth]{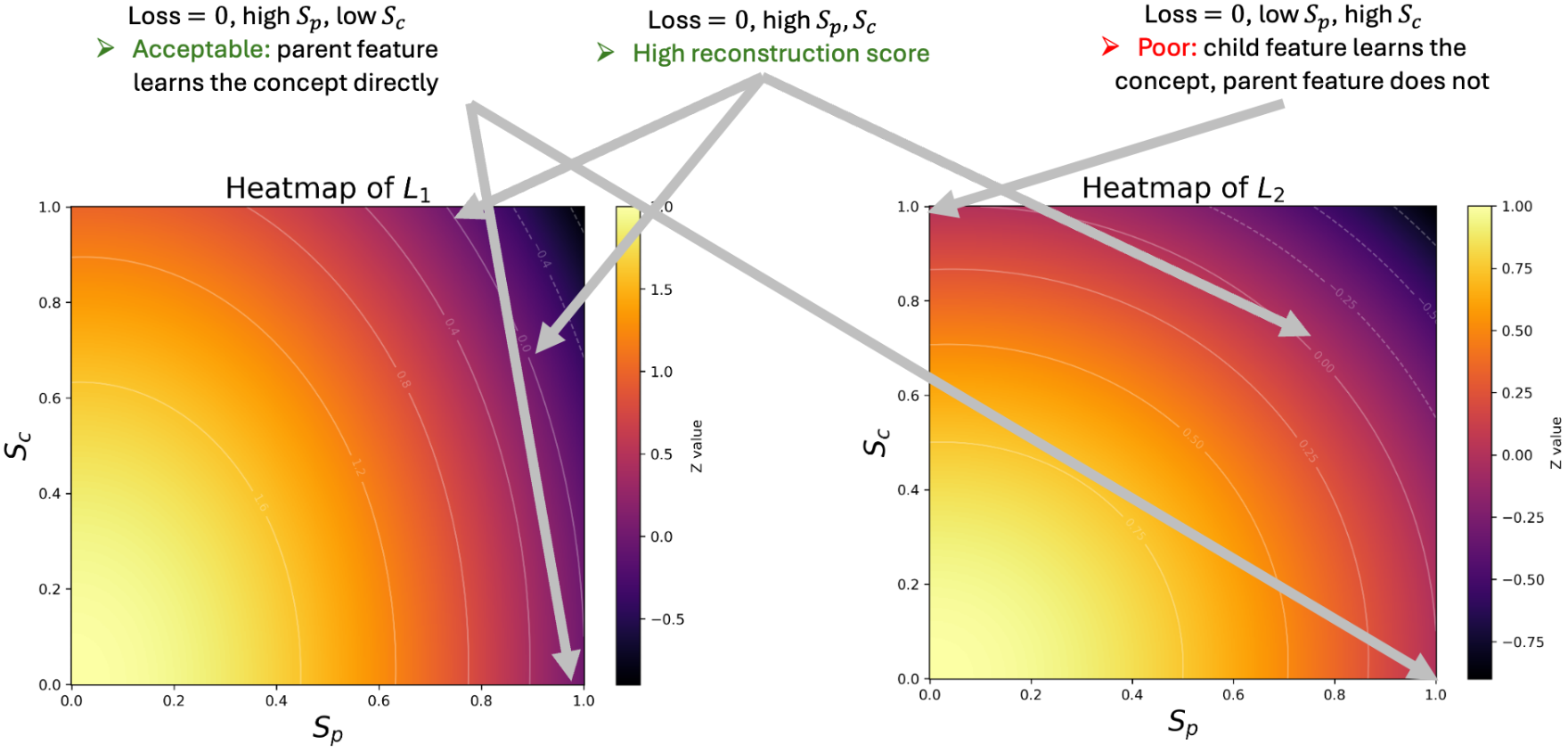}
    \caption{This figure plot the landscape of $L_1,L_2$ when $S_p, S_c \in [0,1]$. \textbf{Left:} $L_1$ leads to either high reconstruction score, or the parent feature learns the child concept directly. \textbf{Right:} $L_2$ can lead to high $S_c$ and low $S_p$ (similar to examples in Figure \ref{fig: non dense feature evidence}), the parent feature does not reconstruct the child concept. Note that $S_p, S_c$ cannot be 1 at the same time (and the loss will be negative), because this means $\mathbf{d}_c=\mathbf{d}_p=\mathbf{d}_c^\ast\rightarrow \mathbf{d}_p^T \mathbf{d}_c=1$, which is impossible due to our loss.}
    \label{fig:loss_landscape}
\end{figure*}

The results show that optimize using $L_1$ can avoid case where the parent feature does not reconstruct the child concept (examples in Figure 1), while the $L_2$ cannot. This directly support our proposition.

\textbf{Interpretation:} This result means that enforcing multi-level reconstruction loss for feature pairs that strictly follow activation coverage (enforced via our tree structure) directly leads to high reconstruction score.

\textbf{Q2: Why child encoder has low cosine similarity with parent feature decoder?:}

\textbf{Propostion 2:} $\mathbf{e}_c^T\mathbf{d}_p\simeq0$. This also shows that we should not use $\mathbf{e}_c$ in our reconstruction score because it cannot separate correct and incorrect pairs.

\textbf{Proof:} We provide proof for $L_1$, proof for $L_2$ is then straightforward. At convergence, $$\partial L_1/\partial \mathbf{e}_p=(4\mathbf{e}_p^T\mathbf{d}_c^\ast-4\mathbf{d}_p^T\mathbf{d}_c^\ast+2k\mathbf{e}_c^T\mathbf{d}_c^\ast)\mathbf{d}_c^\ast=0,$$ and $$\partial L_1/\partial \mathbf{e}_c=(2\mathbf{e}_c^T\mathbf{d}_c^\ast-2\mathbf{d}_c^T\mathbf{d}_c^\ast+2k\mathbf{e}_p^T\mathbf{d}_c^\ast)\mathbf{d}_c^\ast=0.$$ 
Solving these and remove $k^2$, we have $\mathbf{e}_p\simeq \mathbf{d}_p-k\mathbf{d}_c/2$ and $\mathbf{e}_c \simeq \mathbf{d}_c-k\mathbf{d}_p$. Therefore $\mathbf{e}_c^T\mathbf{d}_p\simeq0$. This completes the proof.

\textbf{Interpretation:} The child encoder is not a good approximation of the true concept vector $\mathbf{d}_c^*$ as it avoids aligning with parent feature caused by the default reconstruction loss optimization. This leads to the low cosine similarity between parent encoder the parent feature decoder vector, and we can not use the child encoder as the true direction of the child concept to differentiate between correct and incorrect hierarchical pairs. 

\textbf{Q: Additional Experiment showing that $|k|=|\mathbf{d}_p^T\mathbf{d}_c|$ is small:}

We collect every parent and child pairs that strictly follow activation coverage $S_{cov}=1$ and measure $|\mathbf{d}_p^T\mathbf{d}_c|$. We conduct on Tree SAE $L_0=32$-2 layers and TopK SAE $L_0=32$ in Table \ref{tab:measure k}:

\begin{table}[h]
    \small
    \centering
    \caption{Table shows the value $|k|=|\mathbf{d}_p^T\mathbf{d}_c| \ll 1$ is small.}
    \begin{tabular}{l|c|c}
        \toprule
        SAE & Tree SAE $L_0=32$-2 layers & TopK SAE \\
        \midrule
        $|k|$ & 0.044 & 0.060 \\
        \bottomrule
    \end{tabular}
    \label{tab:measure k}
\end{table}

This supports our claim that $|k|=|\mathbf{d}_p^T\mathbf{d}_c| \ll 1$.

% Q3.

% What happens if we use mean difference?:

% When we consider mean difference of the child concept as the ground truth, we observe that both $\mathbf{d}_p$ and $\mathbf{d}_c$ yields low value (~0.04) for all SAEs and cannot differentiate between correct and incorrect hierarhical pairs.

\section{Cases Where Activation Coverage Fails}
\label{sec: exp activation coverage fail}
We illustrate more examples of activation coverage that lead to unrelated parent and child features, and the reconstruction condition allows coherent hierarchical pairs. Figure \ref{fig: non dense feature evidence 3} shows that feature 16758 of Matryoshka 4 layers $L_0=80$ shares the same meaning with the parent feature 1605, indeed has high correlation, while feature 23918 represents a completely different concept and has low probe correlation despite having 0.98 activation coverage score. Moreover, in Figure \ref{fig: non dense feature evidence 2} with TopK SAE $L_0=32$, we find that all of the features represent specialised cases of the parent feature have high probe correlation, verifying that the reconstruction condition allows intuitive hierarchical feature pairs.

We further show that, in general, the more strongly hierarchical the feature pair, the more they follow the reconstruction condition. We conduct additional experiment on the baseline SAE across all sparsity level. We follow the procedure presented in Section \ref{exp: hierarchy} to measure the rate of both parent and child feature are among the top-5 correlation with the probe weight trained on the child concepts. We randomly sample 100 parents, each with the top-20 child features with the highest activation coverage score. Our intuition is that, the higher the coverage, the higher chance for the pair to be parent and child. Therefore, we want to show that, the higher coverage score, the higher the reconstruction score. We average the score from top-1 to top-20, shown in Figure \ref{fig: top child parent correlation}. The results are consistence across all sparsity level, suggesting that generally, hierarchical pairs follow the reconstruction condition.

%new
\begin{table*}
    \centering
    % \small
    \caption{Average co-occurrence rate of child features in Tree SAE tree structure.}
    \begin{tabular}{lcccc}
        \toprule
         & \multicolumn{4}{c}{Sparsity level ($L_0$)} \\
        \cmidrule{2-5}
        \multicolumn{1}{c}{SAE} & 32 & 48 & 64 & 80 \\
        \midrule
        Tree SAE 2 layers & 0.013 & 0.017 & 0.025 & 0.027 \\
        Tree SAE 4 layers & 0.010 & 0.019 & 0.024 & 0.028 \\
        \bottomrule
    \end{tabular}
    \label{tab: co-occurrence}
\end{table*}

\begin{figure*}
\begin{center}
\includegraphics[width=0.7\textwidth]{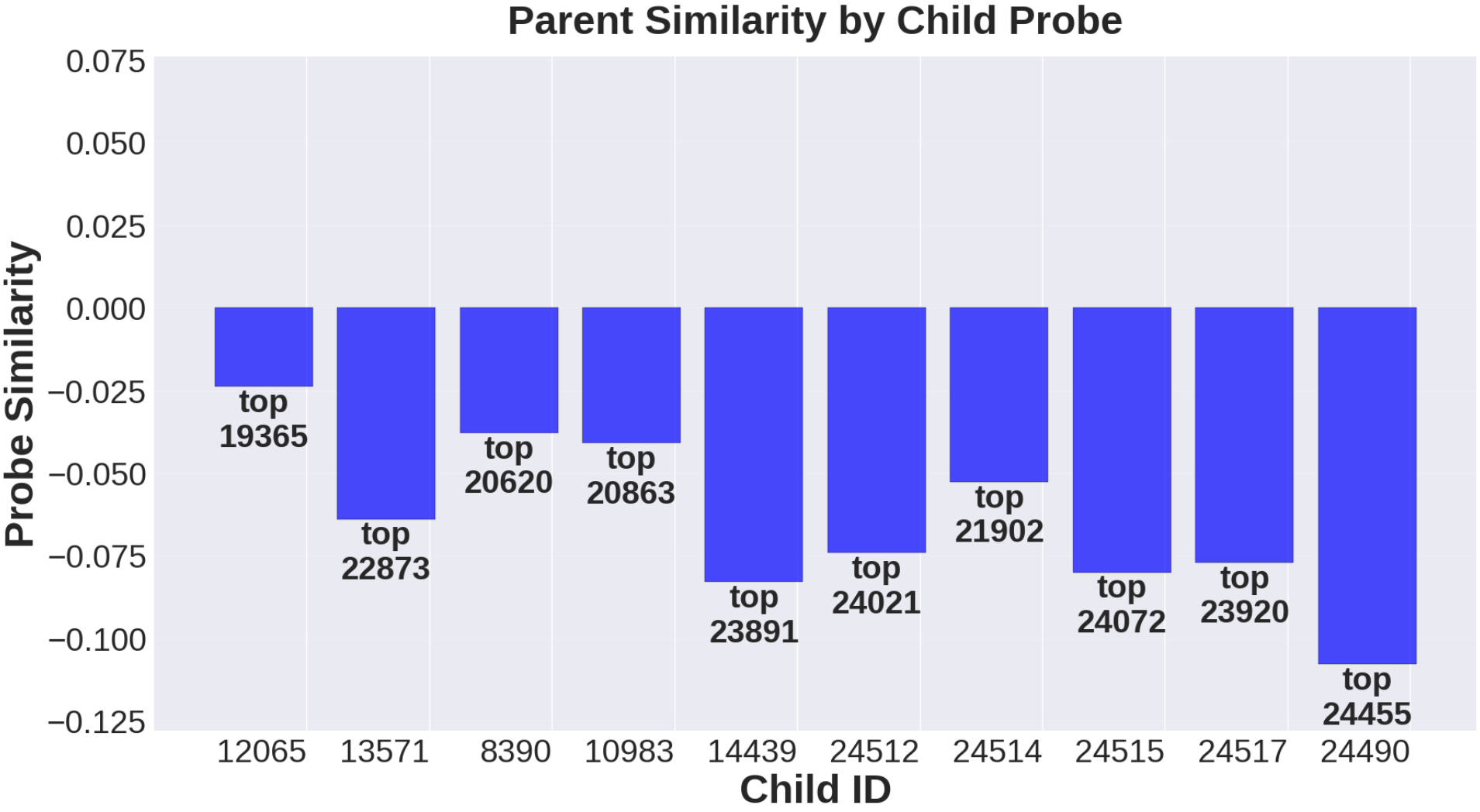} 
\end{center}
\caption{The correlation between the dense ``PCA feature" (feature 3098 - Tree SAE 4 layers $L_0=32$) and 10 probe weights trained on 10 child features activation. In all of the cases, parent features have low correlation, indicating that the parent feature represents a different meaning from all of the child features while having perfect coverage.} 
\label{fig: dense feature evidence}
\end{figure*}

\begin{figure*}
\begin{center}
\includegraphics[width=0.9\textwidth]{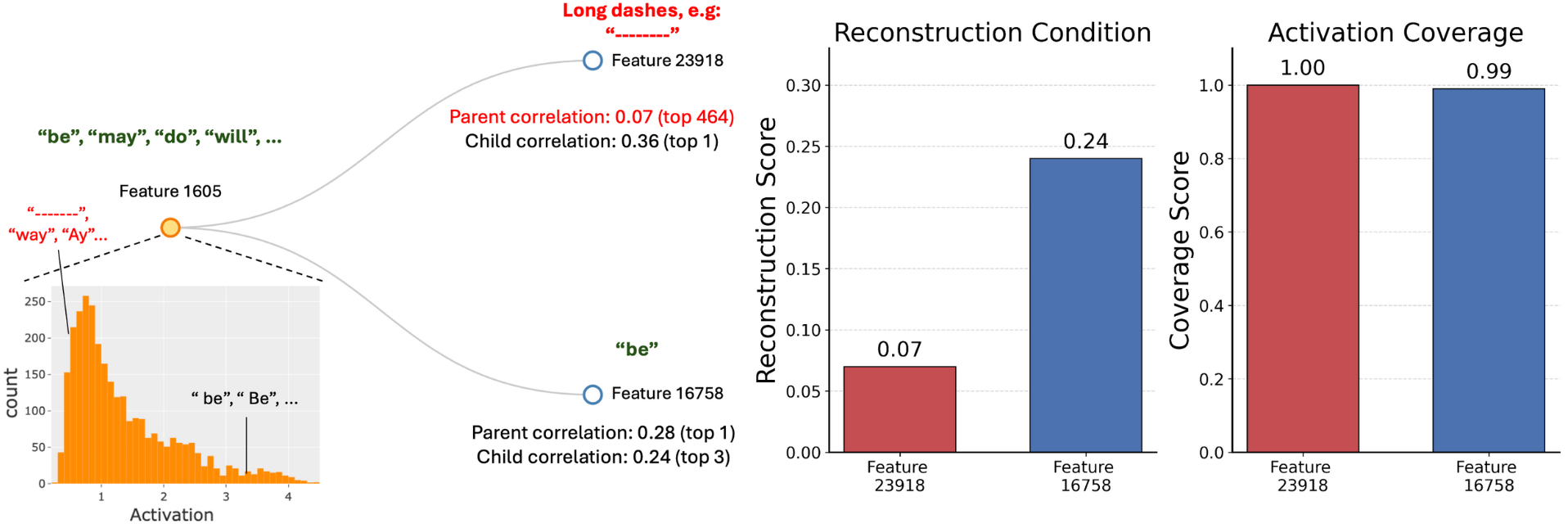} 
\end{center}
\caption{The correlation between non-dense feature (feature 1605 - Matryoshka 4 layers $L_0=80$) and 2 probe weights trained on 2 child features activation with activation coverage scores above 0.98. Child feature 16758, representing a different concept from the rest, corresponding to a spurious low activation value of the parent feature, has significantly lower parent correlation.}
\label{fig: non dense feature evidence 2}
\end{figure*}

\begin{figure*}
\begin{center}
\includegraphics[width=0.9\textwidth]{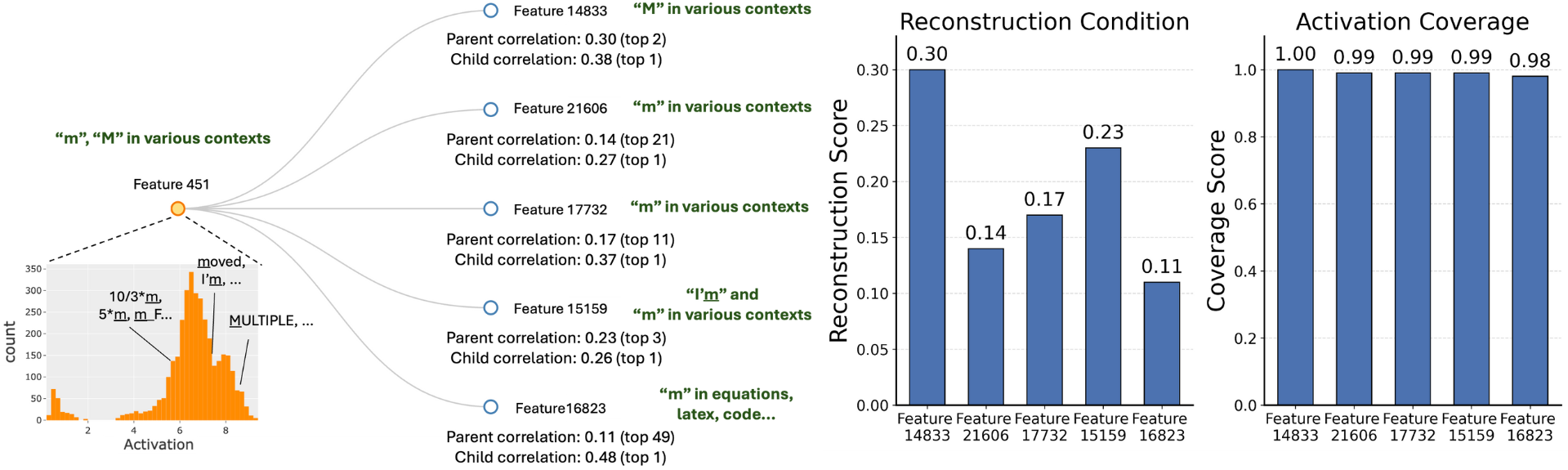} 
\end{center}
\caption{The correlation between the non-dense (feature 451 - TopK $L_0=32$) and 5 probe weights trained on 5 child features activation with activation coverage scores above 0.98. In all of the cases, the parent feature has a high correlation as the meaning of the feature pairs is related. However, feature 16823's probe has slightly lower correlation compared to the rest because it represents a slightly different context from the parent feature.} 
\label{fig: non dense feature evidence 3}
\end{figure*}

%end

\begin{figure*}
\begin{center}
\includegraphics[width=\textwidth]{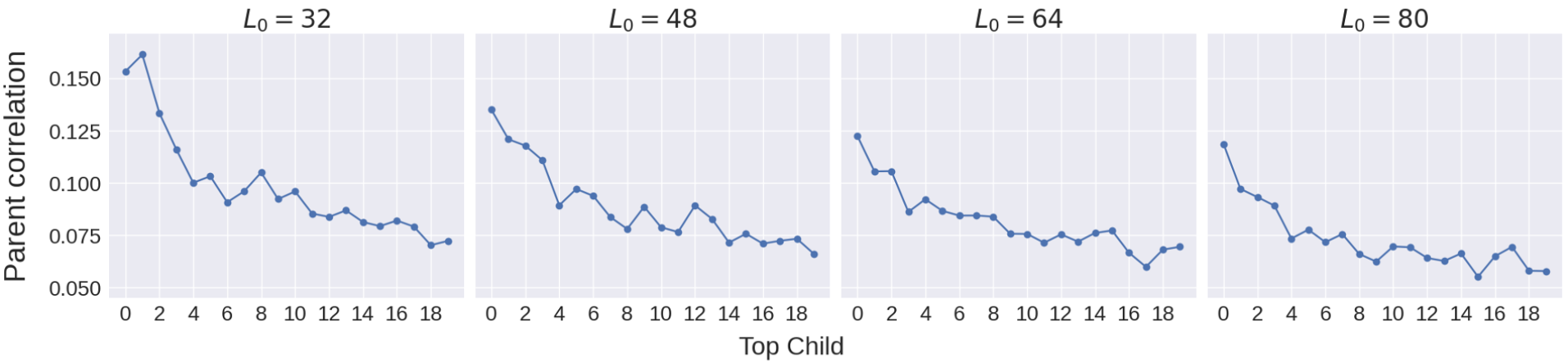} 
\end{center}
\caption{The average of 100 parent features' correlation with probe weight for the top-1 to top-20 child features with the highest activation coverage of TopK SAEs. In general, the higher the coverage score, the higher the probe correlation, suggesting that it is natural for a hierarchical pair to follow the reconstruction condition.} 
\label{fig: top child parent correlation}
\end{figure*}

\section{Necessity of Both Conditions}
\label{sec: necessity of both}
As we have shown in Appendix \ref{sec: exp activation coverage fail}, the activation coverage alone is not enough to find hierarchical pairs, and reconstruction condition correlates well with activation coverage. It can be argued that reconstruction condition is a stricter version of activation coverage and we could use only the reconstruction criterion to search for hierarchical pairs. However, we provide another example based on the non-dense feature (feature 1343 - Tree SAE 2 layers $L_0=32$) shown in Appendix \ref{sec: exp activation coverage fail} that we do need activation coverage. Specifically, we find that feature 4002 has top-2 correlation with the probe weight of feature 7487, compared to the parent feature with top-3 correlation. If we follow the reconstruction condition only, feature 4002 would be more likely to parent feature of 7487. However, feature 4002 activates on `` Mon" or `` Poly" tokens, suggesting that its concept is `` numerical prefixes" instead of representing ``Monday" as in feature 7487 while activating on similar tokens. This hint the insight for the role of each condition, we believe the activation coverage ensures the parent concepts general enough to covers the child concept, while the reconstruction forces the parent concept to specific enough that the meaning of both features are related.

\section{Choice of MCS}
\label{sec: choice of MCS}
Since MCS can have multiple setups, in this section, we toggle the setup to find the best version to benchmark on our main experiment Section \ref{exp: hierarchy}. As described in \cite{matryoshka}, the MCS measure the correlation between the activation of the candidate parent and child feature, specifically in cases where the child feature activates. We can either scale the score by maximum activation of the parent and child to remove spurious activation \cite{matryoshka} or compute the correlation only \cite{monosemanticity}. We propose another setup that, instead of measuring the correlation, we treat the activation of a feature as binary (1 for non-zero activation and 0 for no activation) and measure the correlation on the binary vectors. Thus, we test 4 possible versions (whether to use binary and whether to use scaling) to test the score on our hierarchical metric (Section \ref{exp: hierarchy}. We run the hierarchy experiment on Tree SAEs for all specificity levels. The result in Figure \ref{fig: test_mcs} shows that the best setup is non-scaling-binary, while scaling-correlation yields the worst performance consistently across all sparsity. Therefore, we use the non-scaling-binary for all of the MCS algorithms in our main experiment. Note that the best version is also similar to the definition of \textit{activation coverage} condition in Section \ref{exp: hierarchy}.

\begin{figure*}
\begin{center}
\includegraphics[width=0.7\textwidth]{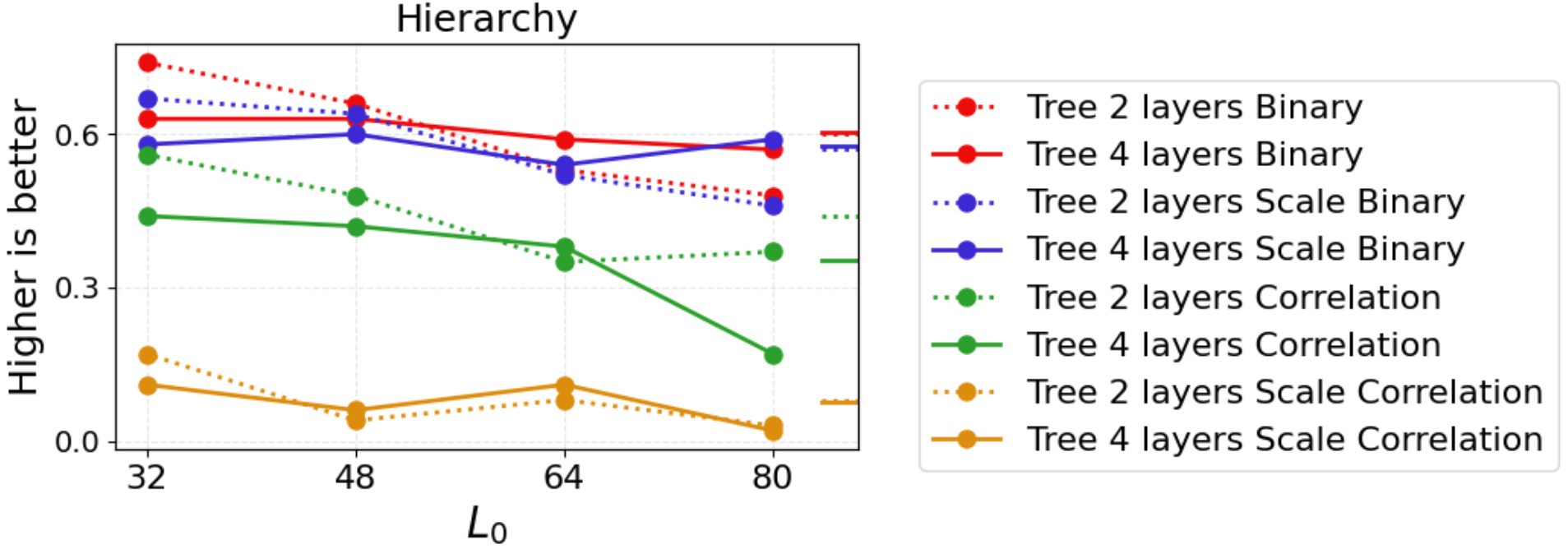} 
\end{center}
\caption{Hierarchy results for Tree SAE on 4 different setups of MCS. The non-scaling-binary consistently outperforms other setups, while scaling-correlation has the worst performance.} 
\label{fig: test_mcs}
\end{figure*}

\section{SAE Training}
\label{sec: sae training}
In this section, we provide the training details of Tree SAE. We train our SAE on layer 5 of GPT2-small \cite{gpt2} with dictionary sizes of 24576 in our main result, and 6144 and 49152 in the scaling experiment. We train with a batch size of 5120, learning rate 1e-4 on 500M tokens of The Pile \cite{thepile}. We employ the Adam optimiser and normalise the gradient as well as the decoder vector at each backward step. These hyperparameters are shared across all SAEs. 

\textbf{Hyperparameters each privilege layer:} We implement two layers and four layers Tree (two or four privilege layers) and Matryoshka SAEs (two or four prefixes \cite{matryoshka}), where the number of feature per layer is $[6144, 18432]$ and $[1536, 3072, 9216, 10752]$ respectively. These numbers were not cherry-picked and are the multiplications of the number of dimensions of the hidden latent in GPT2-small. To ensure fairness in the evaluation, we compute the average $L_0$ of each layer (prefix) in Matryoshka SAEs, then round the number to an integer and set the sparsity level at each Tree SAE layer with the corresponding value. The sparsity levels for all of the SAEs are provided in table \ref{tab: sparsity level main} and \ref{tab: sparsity level scale} for our main and scaling experiments, respectively.

\textbf{Hyperparameters in scaling experiment:} We evaluate the scaling of our SAE on the smallest $L_0$ setup scaled proportionally to dictionary size ($L_0=8$ for dictionary size of 6k and $L_0=64$ for dictionary size of 49k). We also scale the number of features per layer proportionally to the main experiment.

\begin{table*}
    \small
    \centering
    \begin{tabular}{c|c|c|c|c}
        \toprule
        Sparsity level & Matryoshka 2 layers & Tree SAE 2 layers & Matryoshka 4 layers & Tree SAE 4 layers \\
        \midrule
        $L_0=32$ & [26.11, 5.89] & [26, 6] & [20.68, 4.71, 3.79, 2.82] & [21, 5, 4, 2] \\
        $L_0=48$ & [39.70, 8.30] & [40, 8] & [32.22, 6.87, 5.20, 3.71] & [32, 7, 5, 4] \\
        $L_0=64$ & [53.01, 10.99] & [53, 11] & [43.41, 9.02, 6.33, 5.24] & [44, 9, 6, 5] \\
        $L_0=80$ & [68.51, 10.49] & [69, 11] & [55.94, 10.79, 7.51, 5.76] & [56, 11, 7, 6] \\
        \bottomrule
    \end{tabular}
    \caption{Average $L_0$ of each layer in Matryoshka and Tree SAEs main result.}
    \label{tab: sparsity level main}
\end{table*}

\begin{table*}
    \small
    \centering
    \begin{tabular}{c|c|c|c|c}
        \toprule
        Dictionary size & Matryoshka 2 layers & Tree SAE 2 layers & Matryoshka 4 layers & Tree SAE 4 layers \\
        \midrule
        6k & [5.80, 2.20] & [6, 2] & [5.25, 0.41, 0.89, 1.45] & [5, 1, 1, 1] \\
        49k & [55.94, 8.06] & [55, 9] & [51.30, 3.37, 5.22, 4.11] & [51, 4, 5, 4] \\
        \bottomrule
    \end{tabular}
    \caption{Average $L_0$ of each layer in Matryoshka and Tree SAEs in scaling comparison result.}
    \label{tab: sparsity level scale}
\end{table*}

\textbf{Dynamical allocation:} We reallocate the child features of each privilege layer after every 3000 training steps, and increase the interval by two after each allocation, capping the maximum at 10,000 steps. At half of the training (around 50,000 steps), we move all of the remaining dead features to the root node (privilege layer of 0) to allow them to activate freely. We find that this improves the reconstruction and lowers the dead feature rate. The allocation algorithm is shown in Algorithm \ref{alg: dynamic alloc}.

\textbf{Auxiliary loss:} For TopK, Matryoshka, and Tree SAE, we use auxiliary loss. We set the number of auxiliary top-k as 256, and the coefficient is 1/32 as in the original implementation for both Matryoshka and TopK \cite{topk, matryoshka}. We set the same hyperparameter for Tree SAE; however, we only use auxiliary loss for features at the first privilege layer ($\alpha_1 = 1/32; \alpha_{2, 3, 4} = 0$). Our results show that this allows stronger reconstruction quality with a trade-off for dead feature rate; nevertheless, we find that setting auxiliary coefficient as $\alpha_{1, 2, 3, 4} = 1/128$ for other layers can be beneficial on other models such as Gemma-2-2b \cite{gemma_2}.

\textbf{Infrastructure:} Experiments were conducted on A100 GPU. The training Tree SAE with 24k features takes less than 3 GPU-hours by streaming tokens.

\section{Experiment setups}
\label{sec: exp setup}
In this section, we summarise the experiment setups in Section \ref{sec: experiment} for existing metrics \cite{absorption, auto_interp}. 

\subsection{Feature Absorption, Splitting, and Composiion}
\label{sec: absorption and splitting setup}
\textbf{Feature Absorption and Splitting.} We follow the same setup proposed by \cite{absorption} to measure the absorption and splitting rate using first-letter classification tasks.
For feature absorption, \cite{absorption} trains a linear regression to find which direction is responsible for the first letter, then trains a k-sparse probe \cite{ksparse} to select the top features that activate the first letter and calculate the cases where the top features fail to detect the letter. 
For feature splitting, \cite{absorption} again selects the top features for the first letter task using a k-sparse probe, and measures whether adding additional features leads to a significant improvement ($F_1$ score increase of more than 0.03 \cite{absorption}) in detecting the first letter. 
We consider a feature that represents a first-letter concept as having $F_1 > 0.4$ on the task. The remaining setups are similar to those in \cite{saebench}.

% \subsection{Feature Composition}
% \label{sec: composition setup}
\textbf{Feature Composition.} We find the average maximum cosine similarity between the feature direction for all of the features. An SAE has high average cosine similarity indicates that multiple features represent the same information, suggesting feature composition \cite{matryoshka}. 

\subsection{Reconstruction Loss}
\label{sec: reconstruction setup}
Firstly, we compute the downstream reconstruction loss measured cross entropy loss of the substituting the original hidden latent of language model with the reconstruction of the SAEs. Secondly, we compute the variance explained of each token. The lower the cross entropy loss the more information encoded, while the higher the variance explained the better the reconstruction. The results are averaged across 128K tokens, shown in Figure \ref{fig: reconstruction}.

\subsection{AutoInterp}
\label{sec: autointerp setup}
We follow the procedure of AutoInterp \cite{auto_interp, automatically_interp_millions} in measuring the interpretability of the features. Specifically, a large language model is presented with a number of activation examples of a feature, then is asked to predict the rank of different examples by feature activation. We randomly select 200 features to evaluate; the remaining setup is the same as in \cite{saebench}. 